\documentclass[sigconf]{acmart}
\AtBeginDocument{%
	}
\usepackage{algorithmic}
\usepackage[linesnumbered,ruled]{algorithm2e}
\newtheorem{thm}{Theorem}
\newtheorem{ass}{Assumption}
\newtheorem{defi}{Definition}
\newtheorem{lem}{Lemma}
\newtheorem{rmk}{Remark}

\newtheorem{cor}{Corollary}

\DeclareMathOperator{\Var}{Var}

\DeclareMathOperator{\Lap}{Lap}
\DeclareMathOperator{\Clip}{Clip}

\usepackage{enumitem}
\usepackage{subcaption}
\setlist[itemize]{leftmargin=*}

\allowdisplaybreaks

\setcopyright{acmlicensed}
\copyrightyear{2025}
\acmYear{2025}
\acmDOI{10.1145/nnnnnnn.nnnnnnn}
\acmConference[CCS '26]{Computer and Communications Security}{November 15--19,2026}{Hague, Netherland}%{Make sure to enter the correct
\acmISBN{978-1-4503-XXXX-X/2018/06}

\newcommand{\mypara}[1]{\noindent\textbf{#1.}\xspace}
\newcommand{\norm}[1]{\left\lVert#1\right\rVert}
\acmSubmissionID{2532}

\begin{document}
\title{Robust Estimation of Sparse Numerical Vectors under Local Differential Privacy}

\author{Puning Zhao}
\email{zhaopn@mail.sysu.edu.cn}
\affiliation{%
  \institution{Shenzhen Campus of Sun Yat-sen University}
  \country{}
}

\author{Zhikun Zhang}
\affiliation{%
  \institution{Zhejiang University}
  \country{}
}
\email{zhikun@zju.edu.cn}

\author{Shaowei Wang}
\email{wangsw@gzhu.edu.cn}
\affiliation{%
  \institution{Guangzhou University}
  \country{}
}

\author{Sheng Yue}
\email{yuesh5@mail.sysu.edu.cn}
\affiliation{%
  \institution{Shenzhen Campus of Sun Yat-sen University}
  \country{}
}

\author{Bangzhou Xin}
\email{xbw401@gmail.com}
\affiliation{
  \institution{National Interdisciplinary Research Center of Engineering Physics}
  \country{}
}

\author{Tianhang Zheng}
\affiliation{%
  \institution{Zhejiang University}
  \country{}
}
\email{zthzheng@zju.edu.cn}

\author{Pengfei Zhang}
\affiliation{%
  \institution{Anhui University of Science and Technology}
  \country{}
}
\email{zpf.bupt@bupt.cn}

\author{Xiaochun Cao}
\authornote{Corresponding author}
\email{caoxiaochun@mail.sysu.edu.cn}
\affiliation{%
  \institution{Shenzhen Campus of Sun Yat-sen University}
  \country{}
}

\renewcommand{\shortauthors}{P.Zhao, Z.Zhang, S.Wang, S.Yue, B.Xin, T.Zheng, P.Zhang, X.Cao}

	\begin{abstract}
Local differential privacy (LDP) protocols are vulnerable to poisoning attacks. Existing research have proposed efficient defense strategies for single-item users. However, in practice, a user may possess multiple items. The defense against poisoning attacks for multi-item users is challenging, because due to larger output spaces, the adversary can conduct more powerful attacks without being detected. In this paper, we address the robust sparse vector mean estimation problem, in which each user has a vector with $m$ nonzero coordinates. We propose \emph{Randomized Projection with Clipping} (RPC). Firstly, the server sends a random binary vector to each user. The user then projects its local data on the vector, and clip the value to restrict the attacker's capability. To handle clipping bias, we propose a correction method based on a careful analysis that gives an exact expression of the bias. As a result, bias-variance tradeoff is no longer needed, thus the clipping threshold can be further reduced to shrink the output space and enhance robustness. We provide a rigorous theoretical guarantee of the estimation error under all possible attacks. Numerical experiments show that under trusted environments, our new method achieves comparable or better performance than existing methods, indicating that our method is already an efficient estimator in its own right. Under untrusted environments, our method is also significantly more robust to poisoning attacks.
	\end{abstract}


%\ccsdesc[500]{Do Not Use This Code~Generate the Correct Terms for Your Paper}
%\ccsdesc[300]{Do Not Use This Code~Generate the Correct Terms for Your Paper}
%\ccsdesc{Do Not Use This Code~Generate the Correct Terms for Your Paper}
%\ccsdesc[100]{Do Not Use This Code~Generate the Correct Terms for Your Paper}



\maketitle

\section{Introduction}		
Local differential privacy (LDP) \cite{dwork2006calibrating,kasiviswanathan2011can} is the de facto standard for privacy protection of sensitive data. Under LDP, each user randomly perturbs its sample locally. After that, the server collects these perturbed values to conduct statistical analysis. LDP has been applied in many high-tech companies, including Google \cite{erlingsson2014rappor}, Apple \cite{apple}, Samsung \cite{nguyen2016collecting}, and Microsoft \cite{ding2017collecting}. 

A crucial challenge of the deployment of LDP in practice is data poisoning attacks, which means that some dishonest users may send incorrect signals to the server. A survey \cite{thomas2013trafficking} shows that dishonest users come from two sources. Firstly, an adversary can purchase some fake accounts from underground markets with cheap prices. Secondly, the adversary may also maliciously control some users, and modify their signals sent to the server. Unfortunately, LDP protocols are very vulnerable to poisoning attacks \cite{cheu2021manipulation,cao2021data,wu2022poisoning,zhao2025attack}, especially under small privacy budget $\epsilon$. An intuitive explanation is that the definition of LDP requires the outputs to be insensitive to input values. As a result, an adversary can severely alter sample values without being detected by the server. In addition, many LDP protocols use encoding schemes to spread the influence of its sample value. The output space of such encoding may be much larger than the input space, leaving room for adversarial manipulation.

\mypara{Existing Works}
In recent years, researchers have proposed various defense strategies against poisoning attacks. Some earlier works \cite{cao2021data,huang2024ldpguard} attempt to detect malicious users at the server side and then remove them. However, according to high dimensional robust statistics theories \cite{steinhardt2018robust,diakonikolas2023algorithmic}, even for non-private estimation tasks, an attacker can alter the sample value by roughly $\Theta(\sqrt{d})$ without being detected. This issue becomes more serious under LDP \cite{cheu2021manipulation}. There are also some post-processing methods \cite{sun2024ldprecover,wang2019locally} that correct the final results if they appear to be anomalous. However, the adversary can bypass these defenses by refining the attack strategy \cite{li2024robustness}.

To provide a rigorous guarantee of the robustness to poisoning attacks, some research redesign LDP protocols to make them inherently robust to manipulation attacks. \cite{cheu2021manipulation} proposed a robust LDP protocol for mean estimation in $\ell_1$ and $\ell_2$ support under small privacy budget $\epsilon = O(1)$. For each user $i$, the server generates a random signal $\mathbf{S}_i$. After receiving $\mathbf{S}_i$, the user $i$ sends feedback signal based on both $\mathbf{S}_i$ and its local sample $\mathbf{x}_i$. The noise is added in this step to satisfy the $\epsilon$-LDP requirement. Finally, the server aggregates feedback signals from all users and give an estimate. \cite{zhao2025attack} extended the work to all privacy levels. The complexity of predefined signals are carefully tuned based on the privacy budget $\epsilon$, to achieve a good tradeoff between robustness and preserving user's information.

Despite such progress, most existing methods focus only on cases such that each user contains only one item. However, in many real-world applications, a user may contain multiple items. The design of robust estimators for multi-item users is challenging, since the output spaces of local randomizers are much larger than the single-item case. To prevent the attacked users from being detected by the server, the adversary needs to alter the user's feedback signal within the output space of the local randomizer. Consequently, a larger output space indicates that the attack can be made more powerful without being detected. 

%As a result, if we directly extend existing robust LDP methods to multi-item case, then the adversary can alter the outputs of LDP protocols in a larger space, resulting in higher estimation error. 

\mypara{Our Goal} In this paper, we explore robust mean estimation of $m$-sparse vectors under LDP. This task assumes that each user has a vector $\mathbf{x}_i$ of length $d$, among which there are $m$ nonzero components. The mean estimation over sparse numerical vectors is important because it can be applied into the following standard tasks.
\begin{itemize}
	\item Frequency estimation over set-valued data \cite{wang2019locally}. It can be viewed as a special case of sparse vector mean estimation, where each user has a binary vector $\mathbf{x}_i\in \{0,1\}^d$, indicating whether the user owns each of these $d$ items;
	
	\item Key-value data aggregation \cite{ye2019privkv,ye2021privkvm,gu2020pckv}. In this task, each user $i$ contains multiple key-value pairs $(k_{ij}, v_{ij})$, $j=1,\ldots, m$. It can be converted to two sparse vector mean estimation problems, in which the first one estimates the key frequency, and the second one estimates the product of keys and values.
\end{itemize}

Apart from these applications, sparse vector mean estimation is also closely related to frequent itemset mining \cite{wang2018locally}, linear regression \cite{nguyen2016collecting,wang2019sparse}, and model updates in federated learning with restricted communication \cite{konevcny2016federated}. Therefore, it is crucial to design efficient robust LDP mechanisms for sparse vector mean estimation problems.

%\mypara{Existing Works} We provide a brief overview of existing works in two aspects.

%\emph{1) About sparse vector mean estimation.} A simple approach is to randomly select an item from all $m$ items, and then use the single-item methods \cite{ye2019privkv,ye2021privkvm,gu2020pckv}. \cite{wang2018privset} proposes the \emph{Wheel Method}, which achieves optimal statistical rate of frequency estimation over set-valued data for the first time. \cite{wang2023differentially} proposes the \emph{Collision Mechanism}, which extends the previous work to arbitrary sparse vector mean estimation problems. However, the robustness of these methods against poisoning attacks remain unexplored.

%To provide a rigorous guarantee of the robustness to poisoning attacks, one needs to redesign LDP protocols to make them inherently robust to manipulation attacks. \cite{cheu2021manipulation} proposed a robust LDP protocol for mean estimation in $\ell_1$ and $\ell_2$ support for the first  time for small $\epsilon$. \cite{zhao2025attack} extended the work to all privacy levels. However, these works focus only on single-item cases.

%Designing a robust LDP mechanism for sparse vector mean estimation problem is challenging

%\mypara{Existing Works} 

%In general, existing works focus on simple tasks, such as frequency and mean estimation over single-item data. However, real-world data processing tasks can be much more complex. In particular, a user may possess multiple items, and the privacy needs to be protected at user-level.

\mypara{Our Proposal} We propose a new method named \emph{Randomized Projection with Clipping} (RPC). For a dataset with $n$ users, the server generates $n$ random binary vectors $\mathbf{S}_i$, $i=1,\ldots, n$ and assign them to each user. The user then projects its locally stored $m$-sparse vector $\mathbf{x}_i$ onto $\mathbf{s}_i$. As discussed earlier, we are now encountering a larger output space. To be more precise, since each user has $m$ elements, the output range is $m$ times larger than the single-item case. 

To reduce the effect of adversarial manipulation, we clip the output signal into a much smaller region. The clipping operation introduces bias. Existing works \cite{huang2021instance,andrew2021differentially,dong2024almost} use a relatively large clipping threshold to control bias, resulting in a large output space that leaves room for adversarial manipulation. To overcome this issue, we design a bias correction technique based on a careful analysis. After correction, our estimator becomes unbiased. As we have addressed the concern of bias, we can further reduce the clipping threshold to limit the impact of the adversary, thus our estimator is more robust.

Our method is analyzed both theoretically and experimentally. In trusted environment, if values are all within $\{-1,1\}$, our theoretical result shows that our method reaches the same asymptotic theoretical error bound as existing state-of-the-art methods. We then extend the analysis to more general settings that allows general numeric values within $[-1,1]$. In this case, we achieve a better error rate. Numerical experiments also show that our method performs comparably or better than existing methods. Although this work focuses primarily on robustness to poisoning attacks in untrusted environment, these theoretical and numerical results show that our new proposed RPC method is also a competitive mean estimator in trusted environments. In other words, we ensure robustness without sacrificing the performance with clean data. 

We then move on to analyze the performance in untrusted environment. To the best of our knowledge, our work is the first attempt to derive an theoretical upper bound of estimation error over sparse vectors with poisoning attacks. The upper bound is important because it provides a guarantee that our defense can not be circumvented by refining the attack strategy. To test the practical performance, we then design optimal attack strategies for our method and all baseline methods, and run them on various datasets. The results show that compared with other baseline methods, the estimation error of RPC grows much slower with the increase of the number of corrupted users $q$. Therefore, our method is significantly more robust than existing estimators.

%The theoretical result shows that if there are no poisoning attacks, which means that all users are honest, then our method reaches the same asymptotic theoretical error bound with existing state-of-the-art methods. Numerical experiments show that our method actually achieves smaller error than existing methods. These theoretical and experimental results indicate that we achieve better robustness against corruption without sacrificing the performance with honest users. Moreover, when there are adversarial corruptions, we give a rigorous theoretical analysis that gives a bound of the maximum possible estimation error for all attacks. The adversary can not achieve larger error than this bound by refining its attack strategy. Numerical experiments also show that our method exhibits significantly better robustness against poisoning attacks. 

The contributions of this paper are summarized as follows:
\begin{itemize}
	\item We propose RPC for robust sparse vector mean estimation under LDP. In particular, we conduct bias correction based on a careful analysis.
	
	\item We design optimal attack strategies corresponding to RPC as well as existing methods.
	
	\item We conduct theoretical analysis and numerical experiments to verify our methods.
\end{itemize}
The comparison between our method and existing works are summarized in \autoref{tab:compare}. To the best of our knowledge, our work is the first attempt to investigate robust sparse vector mean estimation under LDP. Moreover, we achieve robustness without sacrificing the estimation performance for clean data.

\begin{table}
\caption{Comparison between our method and existing methods. \textmd{The second column refers to the expected $\ell_1$ error when there are no corruption. The third column is the communication complexity when parameters are optimally tuned. The last column is the bound of additional error caused by at most $q$ adversarial samples.}}\label{tab:compare} 
\begin{tabular}{c|c|c|c}
\hline
    Method & Error ($\ell_1$) & Comm. & Robustness\\
     \hline
      PrivKVM \cite{ye2021privkvm} & $O\left(\sqrt{\frac{d^3}{n\epsilon^2}} \right)$ & $O(\log d) $ & Unknown\\
      PCKV \cite{gu2020pckv} & $O\left(\frac{dm}{\sqrt{n\epsilon^2}}\right)$ & $O(d)$ & Unknown\\
      Succinct \cite{zhou2022locally} & $O\left(d\sqrt{\frac{m\ln n}{n\epsilon^2}}\right) $ & $O(\log m)$ & Unknown\\
      Collision\cite{wang2023differentially} &$O\left(d\sqrt{\frac{m}{n\epsilon^2}}\right)$ & $O(\log m)$ & \textcolor{black}{$O\left(\frac{m}{n\epsilon} (d\sqrt{q}+q\sqrt{d})\right)$}\\
      \textbf{RPC(Ours)} & $O\left(d \sqrt{\frac{m}{n\epsilon^2}}\right)$ & $O(1)$ & $O\left(\frac{q\sqrt{dm\ln n}}{n\epsilon}\right)$\\
      \hline
\end{tabular}
%\vspace{-5mm}
\end{table}

%In general, when there are no attacks, RPC is already an efficient estimator in its own right. More importantly, our method has theoretically guaranteed robustness to poisoning attacks. To the best of our knowledge, our work is the first attempt to investigate the attack and defense of robust sparse vector mean estimation under LDP.

\mypara{Roadmap} In \autoref{sec:prelim}, we present some background information. We introduce the problem definition and existing works in \autoref{sec:formulation}. The proposed method is shown in \autoref{sec:method}. We then give a theoretical analysis in \autoref{sec:theory}. \autoref{sec:extension} extends the analysis to general numerical values. \autoref{sec:strategy} discusses optimal attack strategies. Numerical evaluations are presented in \autoref{sec:numerical}. Finally, we give an overview of related work in \autoref{sec:related} and conclude the paper in \autoref{sec:conc}.
\section{Preliminaries}\label{sec:prelim}

\subsection{Definition of LDP}

Under LDP, each user randomly perturbs its sample locally before sending to the server. Denote $Q$ as the perturbation function. Throughout this paper, we use non-interactive LDP, which means that the output of the local randomizer of each user does not depend on other users. The precise definition of LDP is given as follows.
\begin{defi}\label{def:ldp}
	(LDP) Denote $\mathcal{X}$, $\mathcal{Y}$ as the input and output space, respectively. A randomization function $\mathcal{X}\rightarrow \mathcal{Y}$ is $\epsilon$-LDP with respect to $\mathcal{X}$, if for all $\mathbf{x}, \mathbf{x}'\in \mathcal{X}$ and $O\subseteq \mathcal{Y}$,
	\begin{eqnarray}
		\text{P}(Q( \mathbf{x})\in O)\leq e^\epsilon \text{P}(Q( \mathbf{x}')\in O).
	\end{eqnarray}
\end{defi}

The user does not report $\mathbf{x}$ directly to the server. Instead, it only sends $Q(\mathbf{x})$, thus the privacy of users are guaranteed even if the adversary can get access to all output values. If the user is attacked, then it sends to the server an arbitrary signal within $\mathcal{Y}$.

\subsection{Randomization of Numerical Data}\label{sec:randomizer}

Here we introduce several classical unbiased randomizers that maps numerical sample $u\in [-1,1]$ randomly to $\mathcal{Y}\subseteq \mathbb{R}$. These methods will be used later. We use two quantities to evaluate the quality of randomizers:
\begin{eqnarray}
	c_\epsilon:=\max_{y\in \mathcal{Y}} |y|,
	\label{eq:cepsdf}
\end{eqnarray}
and
\begin{eqnarray}
	V_\epsilon = \max_{u\in [-1,1]} \Var[Y|u].
	\label{eq:vepsdf}
\end{eqnarray}
 $c_\epsilon$ measures the size of output space, which determines the effect of poisoning attack. A larger output space implies that the adversary can alter the value more seriously. $V_\epsilon$ measures the variance of the output of LDP mechanisms for honest users. 

\mypara{Laplace mechanism \cite{dwork2006calibrating}} Given $u\in [-1,1]$, the Laplace mechanism generates
$Y=u+W$,
in which $W$ follows Laplace distribution $\Lap(2/\epsilon)$. 

\mypara{Piecewise mechanism (PM) \cite{wang2019collecting}} PM maps input $u\in [-1,1]$ to $[-b,b]$, in which 
$b =(e^\frac{\epsilon}{2}+1)/(e^\frac{\epsilon}{2} - 1)$.
The conditional pdf of the randomizer is
\begin{eqnarray}
	f_Q(y|u) = \left\{
	\begin{array}{ccc}
		p &\text{if} & y\in [l(u), r(u)]\\
		pe^{-\epsilon} &\text{if} & y\in [-b,b]\setminus [l(u), r(u)],
	\end{array}
	\right.
	\label{eq:pmpdf}
\end{eqnarray}
in which
$p = (e^\epsilon - e^\frac{\epsilon}{2})/(2(e^\frac{\epsilon}{2} + 1))$,
and
	$l(u)=\frac{1}{2}(b+1)u-\frac{1}{2}(b-1)$,
	$r(u) =\frac{1}{2}(b+1) u + \frac{1}{2}(b-1)$.

\mypara{Duchi et al.\cite{duchi2013local}} Given $u\in [-1,1]$, this method generates 
\begin{eqnarray}
	Y = \left\{
	\begin{array}{ccc}
		\frac{e^\epsilon +1}{e^\epsilon - 1} &\text{with probability} & \frac{1+u}{2}\\\
		-\frac{e^\epsilon +1}{e^\epsilon - 1} &\text{with probability} & \frac{1-u}{2}.
	\end{array}
	\right.
	\label{eq:duchi}
\end{eqnarray}

According to \autoref{def:ldp}, it can be shown that all these methods satisfy $\epsilon$-LDP. \autoref{tab:compare} lists the values of $c_\epsilon$ and $V_\epsilon$. Detailed discussions are shown in \autoref{sec:discuss}.
\begin{table}
	\caption{Comparison of $c_\epsilon$ and $V_\epsilon$ values of common randomizers.}\label{tab:compare}
	\begin{tabular}{cccc}
		\hline 
		Method & Laplace \cite{dwork2006calibrating} & Piecewise \cite{wang2019collecting} &Duchi et al.\cite{duchi2013local} \\
		\hline
		$c_\epsilon$ & $\infty$ & $\frac{e^{\epsilon/2}+1}{e^{\epsilon/2}-1}$ & $\frac{e^\epsilon+1}{e^\epsilon-1}$\\
		$V_\epsilon$ & $\frac{8}{\epsilon^2}$ & $\frac{4e^{\epsilon/2}}{3(e^{\epsilon/2}-1)^2}$ & $\left(\frac{e^\epsilon+1}{e^\epsilon-1}\right)^2$\\
		\hline
	\end{tabular}
\end{table}
\section{Problem Formulation and Existing Estimators}\label{sec:formulation}
%This section discusses formulations used in this paper, including the threat model. Moreover, we introduce several important existing works on sparse vector mean estimation. We also include some methods for key-value estimation, as it can be easily converted to sparse mean estimation problems. These existing methods can be divided into two types: Sampling based methods, and methods that use the entire vector.

\subsection{Problem Formulation}
We consider the following settings. The dataset contains $n$ users, and the $i$-th user possess a sample vector $\mathbf{x}_i\in \mathcal{X}$, in which
\begin{eqnarray}
	\mathcal{X} = \{\mathbf{x}\in \{-1,0,1\}^d|\norm{\mathbf{x}}_0= m \}.
	\label{eq:support}
\end{eqnarray}
Each sample in $\mathcal{X}$ has $m$ nonzero components. For any $\mathbf{x}_i\in \mathcal{X}$, if the $j$-th component $x_i(j)$ is nonzero, then $x_i(j)\in \{-1,1\}$. This assumption enables us to conduct bias correction. The extension to general numeric values will be discussed in \autoref{sec:extension}. 

The goal is to estimate the sample mean $\mu = (1/n)\sum_{i=1}^n \mathbf{x}_i$. The central aggregator can not access to $\mathbf{x}_i$ directly. Instead, all samples need to pass through a channel satisfying $\epsilon$-LDP, and some of their output values may be corrupted by an adversary.

\begin{rmk}
	In practice, each user may possess different numbers of items, thus the $\norm{\mathbf{x}_i}_0=m$ condition may not hold for all $i$. However, it is not hard to convert all samples to the same sparsity level. Let $m$ be the maximum number of items among all users, and then pad all vectors to $(d+m)$-dimensional vectors that are all $m$-sparse. In particular, for a user $i$ with $\norm{\mathbf{x}_i}_0=m_i<m$, we let $x_i'(j)=x_i(j)$ for $j=1,\ldots, d$, and $x_i'(d+j)=1$ for $j=1,\ldots, m-m_i$, and $x_i'(d+j)=0$ for $j=m-m_i+1,\ldots, m$. Then the $(d+m)$-dimensional vector $\mathbf{x}_i'=(x_i'(1),\ldots, x_i'(d+m))$ is $m$-sparse. The above process can be repeated for all $i$, so that all samples become $m$-sparse.

\color{black}
    If users have significantly different number of items (and such number is unknown), then we may have to set $m$ manually. If $m$ is smaller than the maximum number of items among users, then the truncation will introduce bias. To the best of our knowledge, the appropriate selection of $m$ in this case and the corresponding bias included by such truncation have not been analyzed before. In this paper, we focus on the robustness issue. The problem caused by highly imbalanced users worths further study but is beyond the current scope of our paper.
\color{black}
\end{rmk}
\begin{rmk}
	A special case of the above formulation is frequency estimation over set-valued data \cite{qin2016heavy,wang2018privset}. In the frequency estimation problem, samples never have negative values. For each user $i$, $x_i(j)\in \{0,1\}$, indicating whether item $j$ belongs to user $i$.
\end{rmk}

\subsection{Threat Model}\label{sec:threatmodel}
\mypara{Attacker's goals} In this paper, we consider an \emph{untargeted attack}, which aims at increasing the overall estimation error. We use $\ell_1$ metric. The attacker's goal is to maximize the $\ell_1$ estimation error $\norm{\hat{\mu}-\mu}_1$, in which $\hat{\mu}$ is the estimated mean vector.

\mypara{Attacker's capabilities} Among $n$ users, we assume that the attacker can corrupt up to $q$ users. Our threat model allows the attacker to let a corrupted user send arbitrary signal to the server. However, in practice, to avoid the corruption from being detected by the server, the signal needs to be within the output space $\mathcal{Y}$ of honest users.

In practice, users may be corrupted due to two reasons: the attacker purchases some fake accounts from the underground market, or modify the signal values sent from existing accounts. Therefore, according to some recent works about high dimensional robust statistics \cite{steinhardt2018robust,diakonikolas2023algorithmic}, we discuss the following two models:
\begin{itemize}
	\item Additive contamination model. Under this model, the set $\mathcal{C}$ of corrupted users are fixed in advance.
	
	\item Strong contamination model. Under this model, the adversary can arbitrarily pick $q$ users to construct $\mathcal{C}$.
\end{itemize}

The adversary is more powerful under strong contamination model, since it can select the set of corrupted users arbitrarily. 

\mypara{Attacker's knowledge} We assume that attacker has exact knowledge of the support $\mathcal{X}$, the output space of local randomizers $\mathcal{Y}$, the privacy budget $\epsilon$, and values of all feedback signals $Y_1,\ldots, Y_n$. However, we assume that the server knows nothing about the attacker. Such an asymmetric information is practical, as servers are usually some public institutions, while attackers may be hidden and it is hard to the server to know the attack strategy.

\subsection{Existing Solutions}
\mypara{Sampling-based Estimators} Since a sparse vector like \eqref{eq:support} has multiple nonzero elements, a simple solution is to conduct subsampling and only send one component privately to the server. If the sparsity $m$ is of the same order of magnitude as dimensionality $d$, then one can just sample $j\in [d]$, and send $(j, x_i(j))$ privately to the server. PrivKV \cite{ye2019privkv} and PrivKVM \cite{ye2021privkvm} are two standard methods. On the contrary, if $m\ll d$, then it would be better to randomly sample $k$ from $m$ nonzero components, and send the key-value pair $(k,v)$ to the server, such as PCKV \cite{gu2020pckv}.  While such sampling process makes the privacy protection easier, it causes severe information loss, thus the performance is not satisfactory. We refer to \autoref{sec:sampling} for detailed discussion.

\mypara{Succinct Mechanism \cite{zhou2022locally}} This method samples two hash functions $\mathbf{h}_i:[d]\rightarrow [b]$ and $\mathbf{S}_i: [d]\rightarrow\{-1,1\}$. Then the user generates a response as follows. For $k=1,\ldots, b$,
	$B_{ik} = \sum_{j=1}^d \mathbf{1}(h_i(j)=k) x_i(j) \textcolor{black}{S_i(j)}$, and
		$Y_{ik} = RQ\left(\Clip(B_{ik}, R)/R\right)$,
in which $Q$ can be any randomizer in \autoref{sec:randomizer}, including Laplace mechanism, PM and Duchi et al.\cite{duchi2013local} under $\epsilon$-LDP. Then user sends $\mathbf{h}_i$, $\mathbf{S}_i$ and $Y_{ik}$, $k=1,\ldots, b$ to the server. At the server side, in a trusted environment, the aggregation function is
	$\hat{\mu}_c(j) = (1/n)\sum_{i=1}^n S_i(j) Y_{i, h_i(j)}$.

The value of $R$ is selected as $R\sim \sqrt{m\ln n}$. The bound of $\ell_1$ error becomes $d\sqrt{m\ln n/(n\epsilon^2)}$, which is nearly optimal up to a logarithm factor. 

\noindent \textbf{Collision \cite{wang2023differentially}.} This method was initially designed for frequency estimation over set-valued data, i.e. $x(j)\in \{0,1\}$. To begin with, this method generates $n$ random hash functions $H_i:[d]\rightarrow [t]$ for $i=1,\ldots, n$. For any $\mathbf{x}\in \mathcal{X}$, for convenience, denote 
$H_i(\mathbf{x}) = \{H_i(j)|x(j) = 1\}$.
Then the $i$-th user sends feedback signal $Y_i$ to the server according to the following distribution:
\begin{eqnarray}
	\text{P}(Y_i=j|\mathbf{x}_i) = \left\{
	\begin{array}{ccc}
		\frac{e^\epsilon}{\Omega} &\text{if} & j\in H_i(\mathbf{x}_i)\\
		\frac{\Omega - e^\epsilon c_i(\mathbf{x}_i)}{(t-c_i(\mathbf{x}_i)) \Omega} &\text{if} & j\notin H_i(\mathbf{x}_i),
	\end{array}
	\right.
\end{eqnarray}
for $j=1,\ldots, d$, in which $c_i(\mathbf{x}) = |H_i(\mathbf{x})|$, and $\Omega = me^\epsilon + t-m$ is the normalizer. Recall that our problem setting requires $\norm{\mathbf{x}_i}_0 = m$, thus $c_i(\mathbf{x}_i)\leq m$. It can be shown that the generation of $Y_i$ satisfies $\epsilon$-LDP.

Without adversarial corruption, the aggregator is
\begin{eqnarray}
	\hat{\mu}_c(j) = \frac{\frac{1}{n}\sum_{i=1}^n \mathbf{1}(Y_i=H_i(j)) - \frac{1}{t}}{e^\epsilon/\Omega - 1/t}.
	\label{eq:muccol}
\end{eqnarray}

The $\ell_1$ error of the collision method is $O(d\sqrt{m/(n\epsilon^2)})$ with $\epsilon=O(1)$, which is theoretically optimal.

This method can be extended to include negative values. If $x(j)=-1$ is allowed, then one can just estimate the frequency of both positive and negative elements. To be more precise, the support can be extended to $\{1_-,1_+,2_-,2_+, \ldots,d_-,d_+ \}$, which has $2d$ components. Let $x_e(j_+)=1$ if $x(j)=1$, and $x_e(j_-)=1$ if $x(j)=-1$. We then estimate the mean $\mu_e(1_-,1_+,2_-,2_+, \ldots,d_-,d_+)$ under $\epsilon$-LDP. Finally, the mean of original vector can be calculated by $\mu = (\mu_e(1_+), \mu_e(2_+), \ldots, \mu_e(d_+)) - (\mu_e(1_-), \mu_e(2_-), \ldots, \mu_e(d_-))$.

\mypara{Drawback of existing methods} Due to the information loss caused by the sampling process, all sampling-based estimators fail to achieve optimal estimation error, even if in trusted environments. Succinct mechanism and Collision have significantly reduced the estimation error. However, their robustness to poisoning attacks are still unknown. In numerical experiments in \autoref{sec:numerical}, we show that compared with our new proposed approach, these methods are more vulnerable to poisoning attacks.

\section{Our Proposal}\label{sec:method}

This section discusses the framework of robust estimation. Note that the single-item case has been solved in \cite{cheu2021manipulation}. It is natural to extend this method to the multi-item case.

\subsection{The General Framework}\label{sec:framework}
Each user has a vector of dimensionality $d$, among which $m$ of them are nonzero. The server, users and the adversary engage in the following procedures: (1) The server generates $\mathbf{S}_i\in \{-1,1\}^d$ randomly with uniform probabilities; (2) For $i=1,\ldots, n$, the user calculates $Y_i$ based on $\mathbf{x}_i$ and $\mathbf{S}_i$; (3) The adversary modifies signals:
\begin{eqnarray}
	Z_i = \left\{
	\begin{array}{ccc}
		Y_i &\text{if} & i\notin \mathcal{C}\\
		\star &\text{if} & i\in \mathcal{C},
	\end{array}
	\right.
	\label{eq:modify}
\end{eqnarray}
in which $\star$ denotes arbitrary value determined by the adversary; (4) The server receives signals $Z_1,\ldots, Z_n$ and aggregate them to get the final estimate $\hat{\mu}$. The above framework is illustrated in \autoref{fig:framework}.
\begin{figure}
	\includegraphics[width=\linewidth]{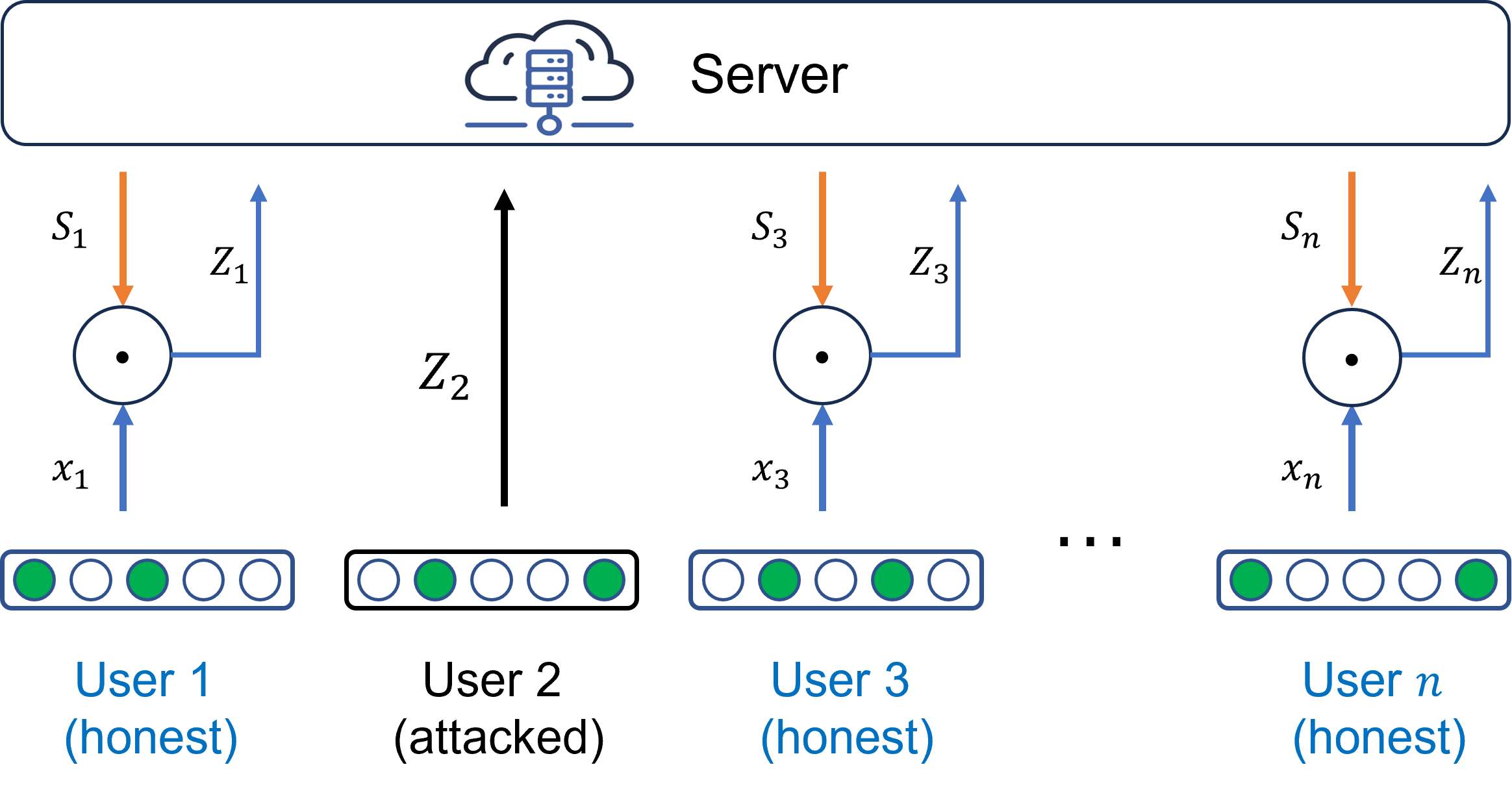}
	\caption{Illustration of the general framework with $d=5$ and $m=2$.}\label{fig:framework}
	\vspace{-3mm}
\end{figure} 

According to the threat model in \autoref{sec:threatmodel}, under additive contamination model, $\mathcal{C}$ is fixed. This corresponds to the case that the adversary purchases some fake accounts to impact the final estimate, while existing accounts are still honest. Under strong contamination model, the adversary can pick $\mathcal{C}$ arbitrarily, as long as $|\mathcal{C}|\leq q$. This corresponds to the case that the adversary can maliciously modify the signals uploaded from users. For $i\in \mathcal{C}$, while our problem setting allows the attacker to modify the signal arbitrarily, $Z_i\in \mathcal{Y}$ needs to be satisfied in practice, in which $\mathcal{Y}$ is the space of $Y_i$, otherwise the attack will be easily detected by the server.

Among the above procedures, step (2) incorporates encoding and perturbation steps, and (4) is the aggregation step. These steps will be specified in the remainder of this section.

\subsection{The Strawman Method}
We begin with the following initial method. Despite not efficient, this method can provide intuition for further improvement. 

\mypara{Encoding} For each user $i$, we just let $u_i=\langle \mathbf{x}_i, \mathbf{S}_i\rangle$. $u_i$ can be viewed as the projection of $\mathbf{x}_i$ on a random vector $\mathbf{S}_i$.

\mypara{Perturbation} Recall that $\mathbf{S}_i\in \{-1,1\}^d$, and from \eqref{eq:support}, $\mathbf{x}_i$ has $m$ nonzero elements within $\{-1, 1\}$, thus $u_i\in [-m, m]$. Note that randomizers in \autoref{sec:randomizer} assume that input values are within $[-1,1]$. Therefore, to make the input fit well in this range, we let
\begin{eqnarray}
	Y_i = mQ\left(u_i/m\right),
	\label{eq:Y}
\end{eqnarray}
in which $Q$ can be either Laplace mechanism, PM or Duchi et al.'s method \cite{duchi2013local}, under $\epsilon$-LDP.

The encoding and perturbation steps correspond to step (2) in \autoref{sec:framework}.

\mypara{Aggregation} The server calculates
\begin{eqnarray}
	\hat{\mu} = \frac{1}{n}\sum_{i=1}^n Z_i \mathbf{S}_i.
\end{eqnarray}
It can be shown that in trusted environments, the estimator is unbiased.

\begin{lem}\label{lem:unbiased}
	Define
		$\hat{\mu}_c = (1/n) \sum_{i=1}^n Y_i \mathbf{S}_i$,
	then
	$\mathbb{E}[\hat{\mu}_c ] = \mu$.
\end{lem}

The proof of \autoref{lem:unbiased} is shown in \autoref{sec:unbiased}. $\hat{\mu}_c$ can be understood as the estimation with clean data. Despite that this framework yields an unbiased estimate of $\mu$ in trusted environments, the variance can be larger than necessary. From \eqref{eq:Y}, recall that we use $V_\epsilon$ to denote the variance introduced by the randomizer $Q$, the variance of $Y_i$ can be $m^2 V_\epsilon$. Moreover, since $Y_i$ can take value from a large range $[-m, m]$. This enables the adversary to alter the value severely without being detected. %Therefore, to reduce the variance, we propose to clip the inner product $u_i=\langle\mathbf{x}_i, \mathbf{S}_i\rangle$.

\subsection{Clipping with Bias Correction}\label{sec:correction}
As discussed earlier, the drawback of the aforementioned strawman method is that the range of encoded variable $u_i$ is too large, and thus a large noise is needed for privacy protection, leading to a suboptimal estimation. Motivated by such drawback, we clip the encoded values. The clipping operation introduces bias. However, we show that the bias can be appropriately corrected.

\mypara{Encoding} For each user $i$, we let
	$u_i = \Clip(\langle \mathbf{x}_i, \mathbf{S}_i\rangle, R)$,
in which $R$ is a hyperparameter, and
$\Clip(u, R) = \max\{-R, \min\{u, R\}\}$.
Now we ensure that the encoded value $u_i$ is always within $[-R,R]$. It is ensured that $R<m$, thus weaker noise is needed to protect the privacy. 

\mypara{Perturbation} Let the user's feedback signal be
	$Y_i = RQ\left(u_i/R\right)$. \textcolor{black}{Here we require that $Q$ is unbiased, and thus $Y_i$ is also unbiased.} Compared with \eqref{eq:Y}, now the variance of $Y_i$ is reduced from $m^2V_\epsilon$ to $R^2 V_\epsilon$.

The clipping operation introduces bias. In many previous works of differential privacy \cite{zhang2022understanding,andrew2021differentially,huang2021instance,dong2024almost}, the common solution is to obtain an upper bound on the bias, and then tune the clipping threshold to achieve a good tradeoff between bias and variance. Under this scheme, to prevent large bias, we have to let the clipping threshold sufficiently large, with introduces large additive noise to meet the LDP requirement. 

To overcome the drawbacks of these methods, we provide an exact calculation of the bias, instead of only calculating the upper bound. To be more precise, we introduce a correction factor $\alpha$. After multiplying $\alpha$, the estimator becomes unbiased in trusted environments.

\mypara{Aggregation} After receiving all signals $Z_1,\ldots, Z_n$, the server calculates 
\begin{eqnarray}
	\hat{\mu} = \frac{\alpha}{n}\sum_{i=1}^n Z_i \mathbf{S}_i,
	\label{eq:rpc}
\end{eqnarray}
in which
\begin{eqnarray}
	\alpha = \frac{1}{1-\sum_{l=0}^{m-1} \binom{m-1}{l} 2^{-(m-1)} C_{l,m, R} },
	\label{eq:alpha}
\end{eqnarray}
in which $C_{l,m,R}=\Clip(2l+1-m, -R+1, R+1)$, with $\Clip(u,a,b) = \max\{a,\min\{u,b \}\}$. We then show the following lemma.

\begin{lem}\label{lem:clip}
	Define
		$\hat{\mu}_c = (\alpha/n) \sum_{i=1}^n Y_i\mathbf{S}_i$,
	Then $\hat{\mu}_c$ is unbiased, i.e. $\mathbb{E}[\hat{\mu}_c] = \mu$.
\end{lem}

The proof of \autoref{lem:clip} is shown in \autoref{sec:clip}. \autoref{lem:clip} shows that for arbitrary clipping threshold $R>1$, with corresponding $\alpha$, the estimator is unbiased. Therefore, we can get rid of the bias-variance tradeoff problem, which is commonly encountered in various statistical problems involving clipping \cite{huang2021instance,andrew2021differentially,dong2024almost}. In other words, we can use a smaller $R$ to reduce the variance and enhance robustness without the concern of clipping bias. Nevertheless, it is also inefficient to make $R$ too small, as $\alpha$ will be large, which increases the variance. We refer to \autoref{sec:theory} for a detailed theoretical analysis.

\subsection{The Complete Algorithm}
Finally, we summarize all discussions above and organize the entire procedures in \autoref{alg:rpc}. 
\begin{algorithm}
	\caption{Random Projection with Clipping}\label{alg:rpc} 
	\textbf{Input:} Samples $\mathbf{x}_1,\ldots, \mathbf{x}_n\in \mathcal{X}$; Privacy budget $\epsilon$; Randomizer $Q$\\
	\textbf{Output:} Estimation output $\hat{\mu}$
	\begin{algorithmic}[1]
		\STATE \underline{\emph{Server:}} Generate $\mathbf{S}_i\in \{-1,1\}^d$ randomly with uniform probabilities, for $i=1,\ldots, n$;\label{step:param}
		\FOR{$i=1,\ldots, n$ in parallel}
		\STATE \underline{\emph{User $i$:}} Calculate $u_i= \Clip(\langle \mathbf{x}_i, \mathbf{S}_i\rangle, R)$; \label{step:encode}
		\STATE $Y_i=RQ(u_i/R)$, in which $Q$ is an LDP protocol satisfying $\epsilon$-LDP \label{step:perturb};
		\ENDFOR
		\STATE \underline{\emph{Adversary:}} Pick the set of corrupted users $\mathcal{C}$ (under strong contamination model only);
		\FOR{$i\in \mathcal{C}$}
		\STATE Adversary determines $Z_i$;
		\ENDFOR 
		\FOR{$i\in [d]\setminus \mathcal{C} $}
		\STATE $Z_i=Y_i$;
		\ENDFOR
		\STATE Calculate $\alpha$ with \eqref{eq:alpha};
		\STATE $\hat{\mu}=\frac{\alpha}{n}\sum_{i=1}^n Z_i \mathbf{S}_i$;\label{step:agg}
		\RETURN $\hat{\mu}$
	\end{algorithmic}
\end{algorithm}
Our method shares some similarity with the succinct mechanism \cite{zhou2022locally}, as both methods project samples onto randomly generated vectors. Our method is different from \cite{zhou2022locally} in mainly two aspects. Firstly, \cite{zhou2022locally} uses two random vectors (which are called hash functions in \cite{zhou2022locally}) $\mathbf{h}_i$ and $\mathbf{S}_i$ for each user $i$. In contrast, we only use $\mathbf{S}_i$, and remove $\mathbf{h}_i$. This enables us to achieve a lower variance. More importantly, in \cite{zhou2022locally}, the clipping threshold needs to be sufficiently large to achieve a good tradeoff between bias and variance. In contrast, we conduct a bias correction to the clipping operation, thus we can reduce the clipping threshold to reduce the noise variance, while still ensuring unbiasedness.

We finally comment on the communication complexity of our method. According to Algorithm \ref{alg:rpc}, each user sends $Y_i=RQ(u_i/R)$ to the server. The communication cost depends on the randomizer $Q$. To reduce the communication cost as much as possible, we can use \eqref{eq:duchi}, such that $Y_i$ only takes binary values. In this case, the communication complexity is $O(1)$, as listed in \autoref{tab:compare}.

\section{Theoretical Analysis}\label{sec:theory}

We conduct theoretical analysis in two steps. In the first step, we analyze the estimation error in trusted environment. In other words, we assume that there are no poisoning attacks. In the second step, we analyze the performance with up to $q$ users being controlled by an adversary. The goal of the first step is to show that the robustness of our estimator is achieved without sacrificing the performance without any corruption. \color{black} Throughout this paper, $\lesssim$ is defined as follows: $a\lesssim b$ if there exists a constant $C$ such that $a\leq Cb$.\color{black}

%We compare with existing methods for sparse numerical vector mean estimation. While our ultimate goal is to construct a robust estimator that can withstand poisoning attacks, we also show that such robustness is achieved without sacrificing the performance if there is actually no adversarial corruption.

\subsection{Estimation Error in Trusted Environments}
 Now we assume that there are no adversarial corruption. All signals are correctly sent to the server, i.e. $Z_i=Y_i$ for all $i$. From \eqref{eq:rpc}, $\hat{\mu}=\hat{\mu}_c$, in which $\hat{\mu}_c$ is the mean estimate with clean data, i.e.
 \begin{eqnarray}
 	\hat{\mu}_c=\frac{\alpha}{n}\sum_{i=1}^n Y_i\mathbf{S}_i.
 	\label{eq:muc}
 \end{eqnarray}

\begin{thm}\label{thm:rpc}
	Under $\epsilon$-LDP,
	\begin{eqnarray}
		\mathbb{E}\left[\norm{\hat{\mu}_c-\mu}_1\right] \leq d\alpha \sqrt{\frac{V_\epsilon R^2+m}{n}}.
        \label{eq:l1}
	\end{eqnarray}
\end{thm}

$V_\epsilon$ is the variance of the randomizer under $\epsilon$-LDP. For all randomizers in \autoref{sec:randomizer}, with $\epsilon=O(1)$, $V_\epsilon = O(1/\epsilon^2)$. Now we pick an specific value of $R$ to get an asymptotic bound of $\ell_1$ error.

\color{black}
Now we discuss the optimal selection of the clipping threshold $R$. \autoref{fig:R}(a) shows the change of $\alpha$ and $\alpha R$ over $R$. $\alpha$ monotonically decreases with $R$, while $\alpha R$ monotonically increases with $R$. According to \eqref{eq:l1}, when $\epsilon$ is small, $V_\epsilon$ is large, then among the expression $V_\epsilon R^2+m$, the first term is dominant, and the final error is nearly proportional to $\alpha R$. Therefore, it would be better to use smaller $R$. On the contrary, with large $\epsilon$, $m$ is the dominant term, then we need to use larger $R$ to reduce the error. \autoref{fig:R}(b) shows the theoretical $\ell_1$ error with $d=20$, $n=1000$ and $m=10$ under various values of $\epsilon$ (i.e. $0.3$, $1$, $2$ and $4$), as well as the corresponding optimal $R$. The results validate that the optimal $R$ increases with $\epsilon$. Moreover, it can be found that the optimal $R$ is close to $\sqrt{m}$ in general, thus we use $R=\sqrt{m}$ to simplify the analysis.
\begin{figure}[t!]
    \centering
    \begin{subfigure}[t]{0.48\linewidth}
        \centering
        \includegraphics[height=1.2in]{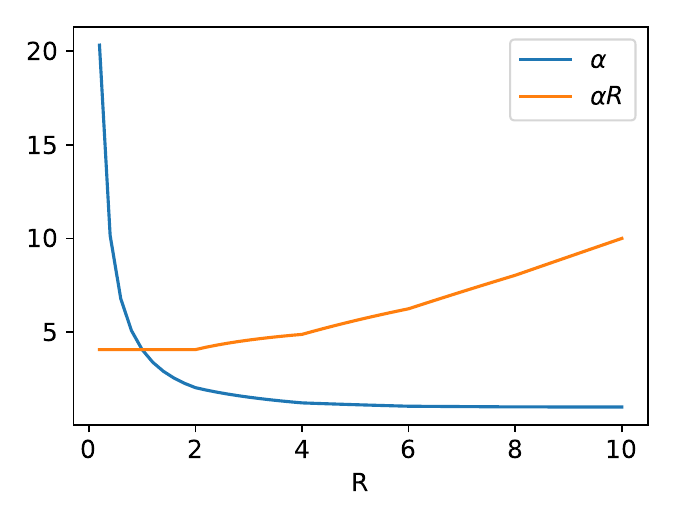}
        \caption{Change of $\alpha$ and $\alpha R$ over $R$.}
    \end{subfigure}%
    \begin{subfigure}[t]{0.48\linewidth}
        \centering
        \includegraphics[height=1.2in]{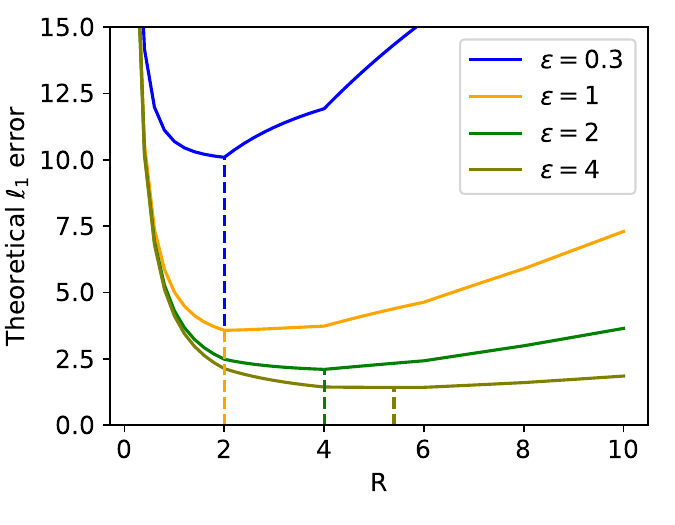}
        \caption{$\ell_1$ error for various $\epsilon$.}
    \end{subfigure}
    \caption{Impact of $R$ on the estimation error}\label{fig:R}
\end{figure}
\color{black}

\begin{cor}\label{cor}
	Let $R=\lceil \sqrt{m}\rceil$. Then 
with $\epsilon=O(1)$,
	\begin{eqnarray}
		\mathbb{E}\left[\norm{\hat{\mu}_c-\mu}_1\right] = O\left(d\sqrt{\frac{m}{n\epsilon^2}}\right).
		\label{eq:l1}
	\end{eqnarray}
\end{cor}

The proof of \autoref{thm:rpc} and \autoref{cor} are shown in \autoref{sec:rpc}. Among all previous methods for honest users \cite{ye2019privkv,ye2021privkvm,gu2020pckv,zhou2022locally,wang2023differentially}, the best result was given in \cite{wang2023differentially}, which achieves an $\ell_1$ bound of $O(d\sqrt{m/(n\epsilon^2)})$ \footnote{\cite{wang2023differentially} gives the error bound in the form of $\ell_2$ instead of $\ell_1$ error. Here we have made a conversion via Cauchy's inequality}. In this work, if all users are honest, then we have also achieved the best $\ell_2$ error. Moreover, it has been proved that the bound $O(d\sqrt{m/(n\epsilon^2)})$ is information-theoretically minimax optimal \cite{wang2023differentially}. Therefore, our estimator is already competitive for honest users in its own right. In other words, we achieve better robustness against poisoning attacks without sacrificing the performance in trusted environments. 

\subsection{Estimation Error in Untrusted Environment}

In an untrusted environment, output signals of LDP protocols may be corrupted. Here we assume that the adversary can alter the values of up to $q$ users. Denote $\mathcal{C}$ as the set of corrupted users. Then
\begin{eqnarray}
	Z_i = \left\{
	\begin{array}{ccc}
		Y_i &\text{if} & i\notin \mathcal{C}\\
		\star &\text{if} & i\in \mathcal{C}.
	\end{array}
	\right.
\end{eqnarray} 
Our threat model allows the adversary to alter output values arbitrarily. However, here we only consider the case such that all altered signals are still within the original output space of LDP protocols, otherwise the attacked signal can be easily detected. Therefore, $Z_i\in \mathcal{Y}$ for all $i$, in which $\mathcal{Y}$ is \textcolor{black}{the space of $Y_i$, which is the output space of randomizer $Q$ scaled by a factor of $R$}. Moreover, the number of corrupted users is at most $q$. Define
\begin{eqnarray}
	\mathcal{Z} := \left\{(Z_1,\ldots, Z_n)|Z_i\in \mathcal{Y}, |\{i|Z_i\neq Y_i\}|\leq q \right\}.
\end{eqnarray}
The adversary can select arbitrary values of $(Z_1,\ldots, Z_n)\in \mathcal{Z}$. The following theorem bounds the $\ell_1$ error under worst-case attack.
\begin{thm}\label{thm:attack}
	Suppose $\epsilon = O(1)$. Let $R=\lceil \sqrt{m}\rceil$. \textcolor{black}{If $c_\epsilon=O(1/\epsilon)$ and $V_\epsilon=O(1/\epsilon^2)$ of randomizer $Q$ are both finite, then} the $\ell_1$ estimation error of \autoref{alg:rpc} can be bounded as follows.
	
	(1) Under additive contamination model, under $\epsilon$-LDP, 
	\begin{eqnarray}
		\mathbb{E}\left[\max_{Z_{1:n}\in \mathcal{Z}}\norm{\hat{\mu}-\mu}_1\right] \lesssim d \sqrt{\frac{m}{n\epsilon^2}} + \frac{q\sqrt{dm}}{n\epsilon}.
		\label{eq:additive}
	\end{eqnarray}
	
	(2) Under strong contamination model, under $\epsilon$-LDP,
	\begin{eqnarray}
		\mathbb{E}\left[\max_{Z_{1:n}\in \mathcal{Z}}\norm{\hat{\mu}-\mu}_1\right] \lesssim d \sqrt{\frac{m}{n\epsilon^2}} + \frac{q\sqrt{dm\ln n}}{n\epsilon}.		
		\label{eq:strong}
	\end{eqnarray}
\end{thm}
\color{black} The conditions in Theorem \ref{thm:attack} require that $c_\epsilon = O(1/\epsilon)$ and $V_\epsilon=O(1/\epsilon^2)$. Piecewise mechanism \cite{wang2019collecting} and Duchi et al. \cite{duchi2013local} both satisfy these requirements. For Laplace mechanism, $c_\epsilon$ is infinite, and thus does not meet these requirements. Intuitively, if the range of perturbed value is infinite, then the attacker can alter the value arbitrarily without being detected.\color{black}  
The first term in \eqref{eq:additive} or \eqref{eq:strong} is caused by the randomness of our algorithm and the LDP mechanism, while the second term is caused by adversarial manipulation. Note that these bounds hold for all attack strategies under corresponding threat models. Therefore, our method has a solid theoretical guarantee that the adversary can not induce an estimation error larger than those indicated in \autoref{thm:attack}. \textcolor{black}{With $q=0$, the $\ell_1$ estimation error is $O(d\sqrt{m/(n\epsilon^2)})$, which matches the minimax lower bound proposed in \cite{wang2023differentially}.}
We also have the following additional remarks.
\begin{rmk}
	The difference between the estimation error under additive and strong contamination model is a $\sqrt{\ln n}$ factor. This is slightly different from traditional robust statistics theory for non-private data \cite{steinhardt2018robust,diakonikolas2023algorithmic}, which has same asymptotic bounds under two models. 
\end{rmk}
\begin{rmk}
	It is observed that with $m=1$, our results under strong contamination model reduces to that in \cite{cheu2021manipulation} and \cite{zhao2025attack}.
\end{rmk}
\begin{rmk}
	Currently, our analysis is conducted under the condition $\epsilon=O(1)$. The algorithms and corresponding analysis with larger $\epsilon$ are left as future works.
\end{rmk}
\color{black}
\begin{rmk}
    A condition in \autoref{thm:attack} is that $c_\epsilon$ and $V_\epsilon$ are both finite. The Laplace mechanism has infinite $c_\epsilon$, and is thus not suitable for robust estimation. Intuitively, with a infinite range of noise, the adversary can manipulate the sample value arbitrarily far away from its ground truth while remaining indistinguishable from benign samples.
\end{rmk}
\begin{rmk}
Now we compare with the Collision mechanism \cite{wang2023differentially}. According to the analysis in \autoref{sec:colanalysis}, under additive contamination model, with optimal parameter $t^*=2m-1+me^\epsilon$ \cite{wang2023differentially},
    \begin{eqnarray}
      \max_{Z_{1:n\in \mathcal{Z}}} \norm{\hat{\mu}-\mu}_1\lesssim d\sqrt{\frac{m}{n\epsilon^2}}+ \frac{m}{n\epsilon} (d\sqrt{q}+q\sqrt{d}).  
\end{eqnarray}
Under strong contamination model, there is an additional $\sqrt{\ln n}$ factor. Compared with \eqref{eq:additive} and \eqref{eq:strong}, it can be shown that our method has a tighter upper bound. 
\end{rmk}
\color{black}

\section{Extension to General Numerical Values}\label{sec:extension}
In previous sections, it is assumed that all samples are within $\mathcal{X}$ specified in \eqref{eq:support}, which requires $x_i(j)\in \{-1,0,1\}$. This setting can be applied in frequency estimation or cases involving only binary responses. However, in many practical applications, the response of users can be more diverse. 

In this section, we generalize the problem setting, such that $x_i(j)$ can take arbitrary value in $[-1,1]$. The corresponding support set is 
\begin{eqnarray}
	\mathcal{X} = \{\mathbf{x}\in [-1,1]^d|\norm{\mathbf{x}}_0= m \}.
	\label{eq:suppgeneral}
\end{eqnarray}
Recall that our bias clipping method is designed only for $x_i(j)\in \{-1,0,1\}$. There are two ideas to extend our method to allow general values in $[-1,1]$. The first one is a indirect approach, which has been commonly used in existing literatures \cite{ye2019privkv,ye2021privkvm,gu2020pckv,wang2023differentially}. It maps each value to $-1$ or $1$, with probabilities assigned to ensure unbiasedness, i.e. 
\begin{eqnarray}
	x_i'(j) = \left\{
	\begin{array}{ccc}
		1 &\text{with probability} & \frac{1}{2}(1+x_i(j))\\
		-1 &\text{with probability} & \frac{1}{2} (1-x_i(j)).
	\end{array}
	\right.
	\label{eq:valrandomize}
\end{eqnarray}
After such mapping, the problem is converted to the estimation with binary values. We can then use \autoref{alg:rpc}.

The second method is a direct approach. Instead of correcting the clipping bias directly as described in \autoref{sec:correction}, now we adjust $R$ carefully to ensure that the bias is sufficiently small.

It remains to compare indirect and direct methods. For the indirect method, the bound of error remains the same as \autoref{thm:rpc} under trusted environment, or \autoref{thm:attack} under untrusted environment. Therefore, in the rest of this section, we merely analyze the direct method. We make the following assumption.
\begin{ass}\label{ass:beta}
	For any $i\in \{1,\ldots, n\}$,
	$\sum_{j=1}^d x_i^2(j)\leq m\beta^2$,
	for some $\beta\in (0,1)$.
\end{ass}
\autoref{ass:beta} requires that the values are no more than $\beta$ on average. A large $\beta$ indicates that most values are close to $1$ or $-1$. A small $\beta$ indicates that values are close to $0$ in general.

\subsection{Performance in Trusted Environments}\label{sec:general-trusted}
Let
	$\hat{\mu}_c = (1/n)\sum_{i=1}^n Y_i \mathbf{S}_i$
be the estimate in a trusted environment. The bias can be bounded with the following lemma.
\begin{thm}\label{thm:trusted}
	The bias of $\hat{\mu}_c = \frac{1}{n}\sum_{i=1}^n Y_i \mathbf{S}_i$ is bounded by
	\begin{eqnarray}
		\norm{\mathbb{E}[\hat{\mu}_c]-\mu}_\infty \leq 2e^{-\frac{(R-1)^2}{2m\beta^2}}.
	\end{eqnarray}
	The mean squared error is bounded by
	\begin{eqnarray}
		\mathbb{E}\left[\norm{\hat{\mu}_c-\mu}_1\right] \leq d\sqrt{\frac{V_\epsilon R^2+m}{n}}+\sqrt{2}de^{-\frac{(R-1)^2}{2m\beta^2}}.
	\end{eqnarray}
\end{thm}
The proof of \autoref{thm:trusted} is shown in \autoref{sec:trusted}. Let $R=\beta\sqrt{m\ln n}+1$, then with $\epsilon = O(1)$,
\begin{eqnarray}
	\mathbb{E}[\norm{\hat{\mu}_c-\mu}_1] = O\left(d \beta \sqrt{\frac{m\ln n}{n\epsilon^2}}+d \sqrt{\frac{m}{n}}\right).
	\label{eq:l1general}
\end{eqnarray}
Recall that \eqref{eq:l1} is the $\ell_1$ bound for cases with values belonging to $\{-1,1\}$. Now we compare \eqref{eq:l1general} with \eqref{eq:l1}. It can be observed that if $\beta\sqrt{\ln n}>1$, then \eqref{eq:l1} is smaller, thus it is better to map numerical values to $1$ or $-1$ following \eqref{eq:valrandomize}. On the contrary, if $\beta\sqrt{\ln n}<1$, then direct estimation yields better results.

\subsection{Performance in Untrusted Environment}
We then discuss the performance under adversarial attack. The results are shown in the following theorem.
\begin{thm}\label{thm:untrusted}
	Suppose $\epsilon = O(1)$. Let $\hat{\mu}=(1/n)\sum_{i=1}^n Z_i\mathbf{S}_i$. Let $R=\beta\sqrt{m\ln n}+1$. The $\ell_1$ estimation error can be bounded as follows.
	
	(1) Under additive contamination model,
	\begin{eqnarray}
		\mathbb{E}\left[\max_{Z_{1:n}\in \mathcal{Z}}\norm{\hat{\mu}-\mu}_1\right]\lesssim d\left(\beta \sqrt{\frac{m\ln n}{n\epsilon^2}}+\sqrt{\frac{m}{n}}\right)+\frac{q\beta \sqrt{dm\ln n}}{n\epsilon}.	\hspace{-10mm}\nonumber\\
		\label{eq:additive-general}
	\end{eqnarray}
	(2) Under strong contamination model,
	\begin{eqnarray}
		&&\mathbb{E}\left[\max_{Z_{1:n}\in \mathcal{Z}}\norm{\hat{\mu}-\mu}_1\right]\nonumber\\
		&& \lesssim d\left(\beta \sqrt{\frac{m\ln n}{n\epsilon^2}}+\sqrt{\frac{m}{n}}\right)+\frac{q\beta \sqrt{dm\ln^2 n}}{n\epsilon}.
		\label{eq:strong-general}	
	\end{eqnarray}
\end{thm}

Compared with \eqref{eq:additive} and \eqref{eq:strong} in \autoref{thm:attack}, the second terms in both \eqref{eq:additive-general} and \eqref{eq:strong-general} have an additional $\beta\sqrt{\ln n}$ factor. We therefore get similar findings to the trusted environment: when $\beta\sqrt{\ln n}>1$, the indirect method is better; when $\beta\sqrt{\ln n}<1$, the direct method is more preferred.

Combining the results of both trusted and untrusted enviroments, these theoretical analyses show that the indirect method is more suitable for cases where most nonzero values are close to $1$ or $-1$, while the direct method is suitable for cases where values are close to $0$. 

\section{Attack Strategies}\label{sec:strategy}
This section discusses the attack strategy on existing methods and our new proposed RPC method. We discuss targeted attack first. Suppose that the adversary aims to increase the estimation within $T\subseteq [d]$. The adversary modifies the result $Y_1,\ldots, Y_n$ to $Z_1,\ldots, Z_n$, with $Z_i\neq Y_i$ for at most $q$ samples. 

\subsection{Attack under Additive Contamination Model}
Under additive contamination model, the adversary can not determine the set of corrupted users $\mathcal{C}$ by itself. It can only decide $Z_i$ to maximize the gain of estimation within $T$. The selection of $Z_i$ needs to ensure that $Z_i\subseteq \mathcal{Y}$, which means that the altered value is still within the output space, so that the attacked feedback signal $Z_i$ can not be detected by the server.

\mypara{Attack on PCKV \cite{gu2020pckv}} Recall from \autoref{sec:sampling} that the range of user response $\mathbf{Y}_i\in \{-1,0,1\}^d$. By \eqref{eq:pckv},
\begin{eqnarray}
	\sum_{j\in T} \hat{\mu}(j)-\sum_{j\in T} \hat{\mu}_c(j)= \frac{m}{na(2p-1)}\sum_{i\in \mathcal{C}} \sum_{j\in T}(Z_i(j)-Y_i(j)).
	\label{eq:gain_pckv}
\end{eqnarray}
Under additive contamination model, we can simply let $Z_i(j) = 1$ for all $i\in \mathcal{C}$ and $j\in T$.

\mypara{Attack on Collision \cite{wang2023differentially}}  After adversarial attack, the estimate becomes
\begin{eqnarray}
	\hat{\mu}_j = \frac{\frac{1}{n}\sum_{i=1}^n \mathbf{1}(Z_i=H_i(j)) - \frac{1}{t}}{e^\epsilon/\Omega - 1/t}.
	\label{eq:mucol}
\end{eqnarray}
Suppose that the adversary hope to increase the estimate for all $j\in T$, i.e. to maximize $\sum_{j\in T} (\hat{\mu}(j) - \hat{\mu}_{c}(j))$. By \eqref{eq:mucol} and \eqref{eq:muccol}, 
\begin{eqnarray}
	&&\sum_{j\in T} (\hat{\mu}(j)-\hat{\mu}_c(j))\nonumber\\
	&&\hspace{-5mm}= \frac{1}{n(\frac{e^\epsilon}{\Omega} - \frac{1}{t})}
	\sum_{i\in \mathcal{C}}\sum_{j\in T} (\mathbf{1}(Z_i=H_i(j)) - \mathbf{1}(Y_i=H_i(j))),\nonumber\\
	\label{eq:gain_coll}
\end{eqnarray}
in which $\mathcal{C}$ is the set of corrupted users. Therefore, the objective $\sum_{j\in T} (\hat{\mu}(j) - \hat{\mu}_{c}(j))$ is maximized when
\begin{eqnarray}
	Z_i = \arg\max_k \sum_{j\in T} \mathbf{1}(H_i(j) = k).
\end{eqnarray}

\mypara{Attack on Succinct mechanism \cite{zhou2022locally}} The overall gain of attacks is
\begin{eqnarray}
	&&\sum_{j\in T} \hat{\mu}(j) -\sum_{j\in T} \hat{\mu}_c(j)
	= \frac{1}{n}\sum_{i\in \mathcal{C}}\sum_{j\in T} S_i(j) (Z_{i, h_i(j)} - Y_{i, h_i(j)})\nonumber\\
	&=& \frac{1}{n} \sum_{i\in \mathcal{C}}\left[\sum_{k=1}^b \sum_{j\in T} \mathbf{1}(h_i(j)=k) S_i(j) Z_{ik}-\sum_{j\in T} S_i(j) Y_{i, h_i(j)}\right]. \nonumber\\
	\label{eq:gain_succ}
\end{eqnarray}
Therefore,
\begin{eqnarray}
	Z_{ik} = \left\{
	\begin{array}{ccc}
		c_\epsilon R &\text{if} & \sum_{j\in T} \mathbf{1}(h_i(j) = k)S_i(j)>0\\
		-c_\epsilon R &\text{if} & \sum_{j\in T} \mathbf{1}(h_i(j)=k)S_i(j) <0.
	\end{array}
	\right.
\end{eqnarray}
If $\sum_{j\in T} \mathbf{1}(h_i(j)=k)S_i(j) =0$, then $Z_{ik}$ can take any values.

\mypara{Attack on our new method} Recall that $\hat{\mu}_c = (\alpha/n) \sum_{i=1}^n Y_i\mathbf{S}_i$. From \eqref{eq:rpc},
\begin{eqnarray}
	\sum_{j\in T} (\hat{\mu}(j) -\hat{\mu}_c(j)) = \frac{\alpha}{n} \sum_{i\in \mathcal{C}} \sum_{j\in T} (Z_i-Y_i)S_i(j).
	\label{eq:gain_new}
\end{eqnarray}
 Recall that $Z_i$ takes value from $[-c_\epsilon R, c_\epsilon R]$. Therefore, for any $i\in \mathcal{C}$, the optimal value of $Z_i$ is
\begin{eqnarray}
	Z_i = \left\{
	\begin{array}{ccc}
		c_\epsilon R &\text{if} & \sum_{j\in T} S_i(j) >0\\
		-c_\epsilon R &\text{if} & \sum_{j\in T} S_i(j) <0.
	\end{array}
	\right.
	\label{eq:Zi}
\end{eqnarray}
If $\sum_{j\in T} S_i(j)=0$, then $Z_i$ can be arbitrary value.

\subsection{Attack under Strong Contamination Model}
The adversary is more powerful under strong contamination model. In this case, it can also pick users to attack. To determine the attack strategy, we calculate the gain $G_i$ of attacking each user $i$.

For PCKV, the gain of attacking user $i$ is
\begin{eqnarray}
	G_i = \frac{m}{na(2p-1)} \sum_{j\in T} (1-Y_i(j))
\end{eqnarray}

For Collision, the gain is
\begin{eqnarray}
	G_i &=& \frac{1}{n\left(\frac{e^\epsilon}{\Omega} - \frac{1}{t}\right)} \left[ \max_k \sum_{j\in T}\mathbf{1}(H_i(j)=k)\right.\nonumber\\
	&&\left.- \sum_{j\in T} \mathbf{1}(H_i(j)=Y_i)\right].
\end{eqnarray}

For Succinct Mechanism, the gain is
\begin{eqnarray}
	G_i &=& \frac{1}{n}\sum_{k=1}^b \left|\sum_{j\in T} \mathbf{1}(h_i(j) = k)S_i(j)\right|\nonumber\\
	&&-\frac{1}{n}\sum_{j\in T} S_i(j) Y_{i, h_i(j)}.
\end{eqnarray}

For our new proposed RPC method, according to \eqref{eq:gain_new} and \eqref{eq:Zi}, the effect of manipulating $Z_i$ on the objective $\sum_{j=1}^T (\hat{\mu}(j)-\hat{\mu}_c(j))$ is
\begin{eqnarray}
	G_i = \frac{\alpha}{n} \left[c_\epsilon R\left|\sum_{j\in T} S_i(j)\right| - Y_i \sum_{j\in T} S_i(j)\right].
\end{eqnarray}
Therefore, the adversary can sort all users according to the above quantity, and then let $\mathcal{C}$ be the set of $q$ users with highest values of $G_i$.

The above steps can be used to construct targeted attack, aiming at maximizing the estimation within $T\subseteq [d]$. To construct an optimal untargeted attack, we can just maximize the gain (i.e. \eqref{eq:gain_pckv},\eqref{eq:gain_coll},\eqref{eq:gain_succ}, and \eqref{eq:gain_new}) over all possible targets $T\subseteq [d]$.
\section{Numerical Experiments}\label{sec:numerical}
\subsection{Experiment Setup}\label{sec:setup}

\mypara{Baseline methods} We compare with the following methods:
\begin{itemize}
	\item PCKV \cite{gu2020pckv}. We run both PCKV-UE and PCKV-GRR. These two methods have the same framework, but the perturbation steps are different. We refer to Algorithm 2 and 3 in \cite{gu2020pckv} for detailed algorithms.
	\item Succinct mechanism \cite{zhou2022locally}. There is an optional clipping operation in \cite{zhou2022locally}, depending on whether we pure or approximate LDP is used. Since we use pure LDP for all other algorithms, we include the clipping step for a fair comparison. For detailed procedures, we refer to Algorithm 1 in \cite{zhou2022locally}.
	\item Collision \cite{wang2023differentially}. We use the optimal parameters suggested in \cite{wang2023differentially}, i.e. $t=me^\epsilon + 2m - 1$.
\end{itemize}

\mypara{Datasets} Firstly, we run experiments using synthesized data, in order to show the impact of $m$ on the estimation error, and validate the advantage of our method for general numerical values. Synthesized datasets are constructed as follows.
\begin{itemize}
	\item Synthesized data for frequency estimation. We fix the number of users be $n=10,000$. The dimensionality is fixed at $d=100$. Each user draws $m$ items from $\{1,\ldots, d\}$ without replacement. To evaluate the impact of $m$ on the estimation error of each method, we run experiments with $m=5,10,20,50$ separately.
	\item Synthesized data for mean estimation with values in $\{-1,1\}$. For this dataset, keys are randomly taken in the same way as the frequency estimation task. We then assign values based on keys. For a key $k\in [d]$, if $k\leq d/2$, then the corresponding value is $-1$, otherwise the value is $1$. Therefore, the ground truth is $\mu(j) = -1$ for $j\leq d/2$, and $\mu(j)=1$ for $j>d/2$.
	\item Synthesized data for mean estimation with continuous numerical values. In this experiment, we still generate keys in the same way. The values follow normal distribution $\mathcal{N}(0, \sigma^2)$, truncated to $[-1,1]$. Here we let $\sigma = 0.2$.
\end{itemize}

We then select four real datasets to compare practical performances of all methods. 
\begin{itemize}
	\item Commerce \cite{commerce}. This is a women clothing e-commerce dataset that has ratings provided by customers. The ratings are from $1$ to $5$. We normalize the ratings by dividing $5$. This dataset has $n=23,486$ users in total, and each user has only one item. The dimensionality is $d=1,206$. 
	\item Clothing \cite{clothing}. This is a dataset describing about whether clothes are fit to customers. The feedbacks include 'small', 'fit' and 'large'. We assign value $1$ for 'fit', and $0$ for 'small' and 'large'. The dataset contains $n=105,508$ users. Most users have no more than $2$ items. For users having more than $2$ items, we randomly select two items for calculation.
	\item Amazon \cite{amazon}. This is a dataset that contains the ratings on products. There are $1,210,271$ users in total. Again, for users with more than $2$ items, we randomly select two for calculation.
	\item Movie \cite{movie}. This dataset contains user ratings on movies. There are $138,493$ users in total. For this dataset, the maximum number of items is set to be $100$.
\end{itemize}

\color{black}
\mypara{Randomizers} Among all randomizers $Q$ including Laplace \cite{dwork2006calibrating}, Piecewise\cite{wang2019collecting} and Duchi et al.\cite{duchi2013local},  we pick the randomizer that achieve minimum error, for a fair comparison between our method and earlier methods. In general, when $\epsilon$ is small, Duchi's two point method is better; when $\epsilon$ is large, the piecewise mechanism is better. This can be observed from $V_\epsilon$ in Table 2. If the ratio of corrupted data is large, then Duchi et al.'s method is better, since $c_\epsilon$ is smaller and the algorithm is more robust to attacks.
\color{black}

For all experiments in this paper, we use a simple correction technique. Since the real frequency satisfies $0\leq \mu(j)\leq 1$, we clip the final results with $\hat{\mu}(j)\leftarrow \Clip(\hat{\mu}(j), 0, 1)$. For the mean estimation problem, the ground truth satisfies $\mu(j)\in [-1,1]$, thus we clip the final results with $\hat{\mu}(j)\leftarrow \Clip(\hat{\mu}(j), 1)$.

For all above tasks, we use the mean absolute error (MAE) as the evaluation metric, i.e. $\norm{\hat{\mu}-\mu}_1 / d$. The value of privacy budget $\epsilon$ ranges from $0.5$ to $2.5$.

\subsection{Experiments under Trusted Environments}
We first run experiments without introducing any attacks. The purpose is to show that our RPC method is already a competitive estimator in its own right. In other words, we significantly enhance the robustness without sacrificing the performance in trusted environments.

In all the following figures, black curves represent the results of PCKV-UE \cite{gu2020pckv}. Orange curves correspond to Collision or Coco \cite{wang2023differentially}. Green curves correspond to our new proposed RPC method. Each point denotes the MAE over $M = 500$ random trials. 

\begin{figure*}
	\includegraphics[width=\linewidth]{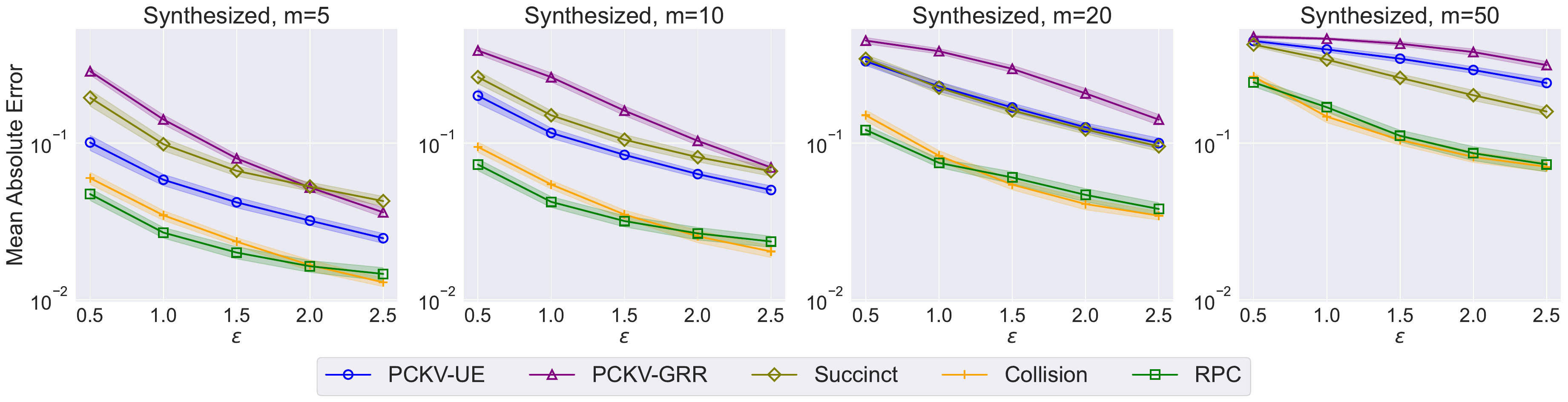}
	\caption{The mean absolute error of frequency estimation in a trusted environment using synthesized data.}\label{fig:freq}
\end{figure*}
\begin{figure*}
	\includegraphics[width=\linewidth]{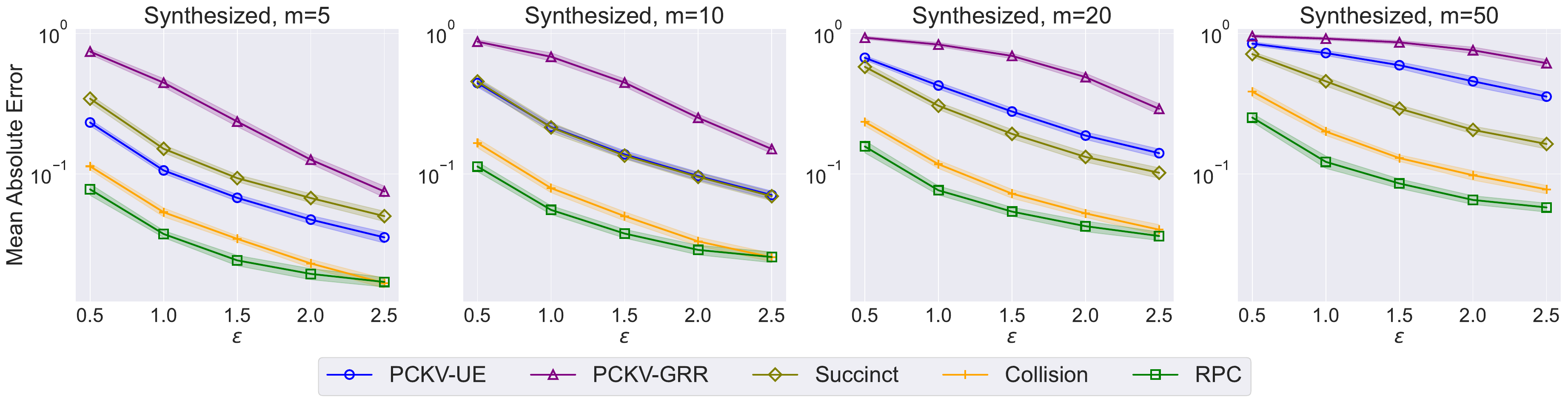}
	\caption{The mean absolute error of mean estimation in a trusted environment, with values being $1$ or $-1$.}\label{fig:mean}
\end{figure*}
\begin{figure*}
	\includegraphics[width=\linewidth]{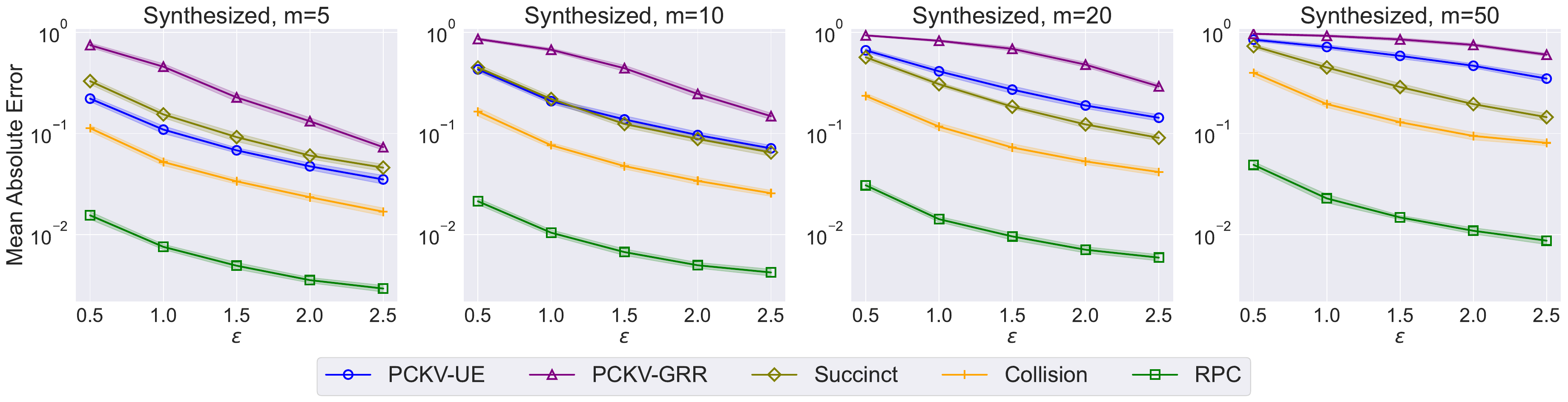}
	\caption{The mean absolute error of mean estimation in a trusted environment, with values following normal distribution with $\sigma=0.2$, truncated to $[-1,1]$. \textcolor{black}{For our method, we use the direct approach in \autoref{sec:extension}, with $\beta = 0.3$.}}\label{fig:meanimb}
\end{figure*}

\begin{figure*}
	\includegraphics[width=\linewidth]{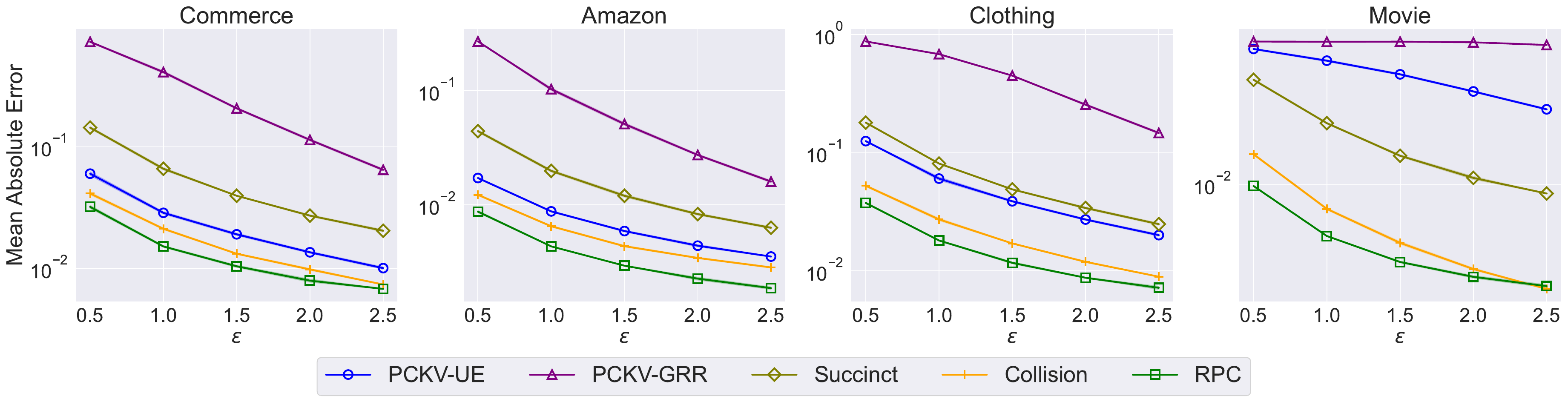}
	\caption{The mean absolute error of mean estimation in a trusted environment for four real datasets.}\label{fig:meanreal}
\end{figure*}
\begin{figure*}[h!]
	\centering
	\includegraphics[width=\linewidth]{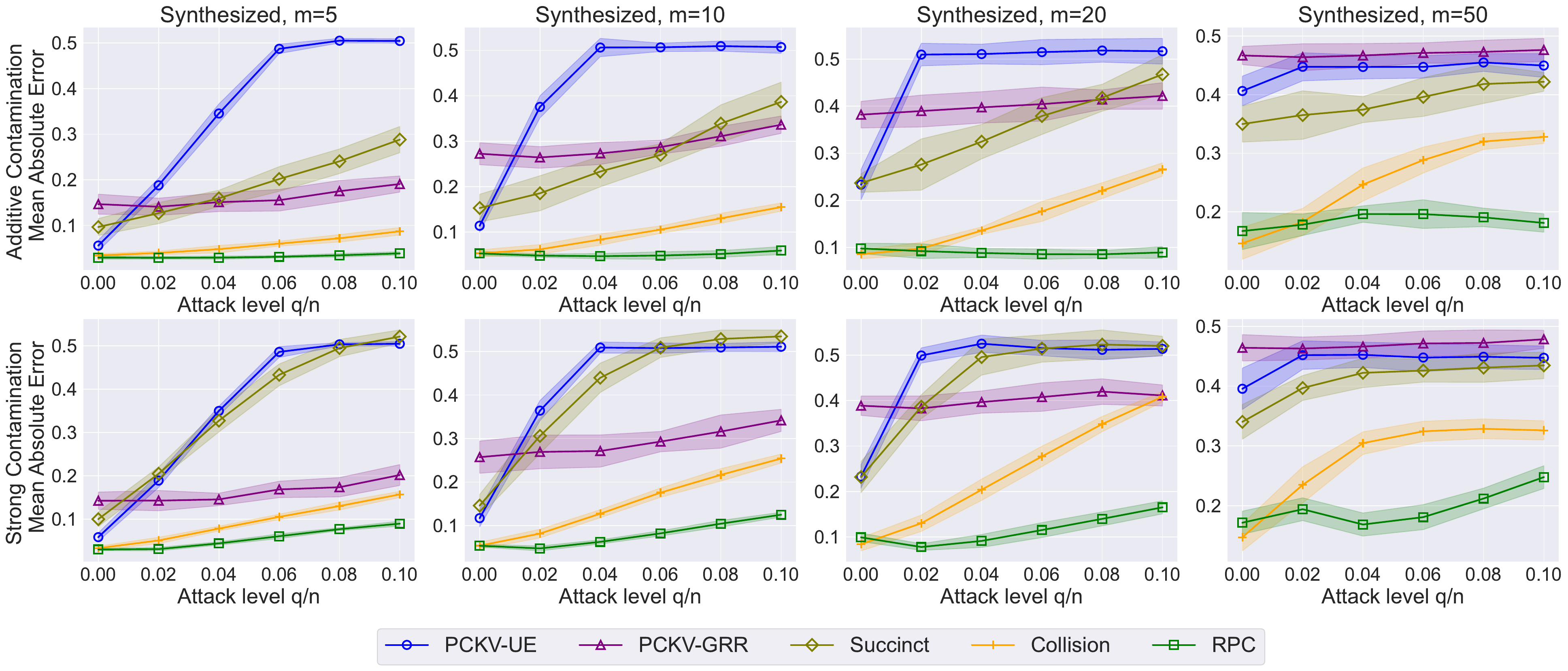}
	\caption{The mean absolute error of frequency estimation in untrusted environment using synthesized data.}\label{fig:freq_attack}
\end{figure*}

\begin{figure*}[h!]
	\centering
	\includegraphics[width=\linewidth]{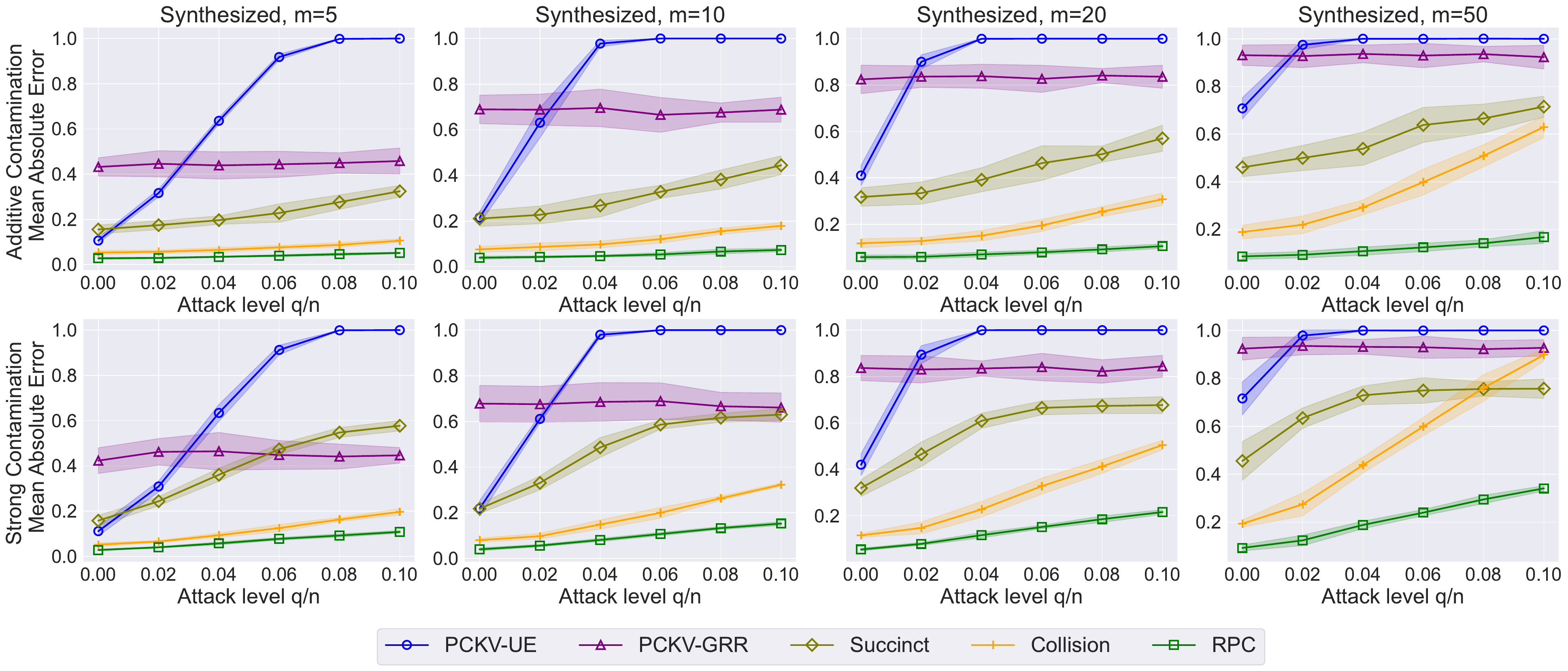}
	\caption{The mean absolute error of mean estimation in untrusted environment using synthesized data.}\label{fig:mean_attack}
	%\vspace{-3mm}
\end{figure*}

\begin{figure*}
	\centering
	\includegraphics[width=\linewidth]{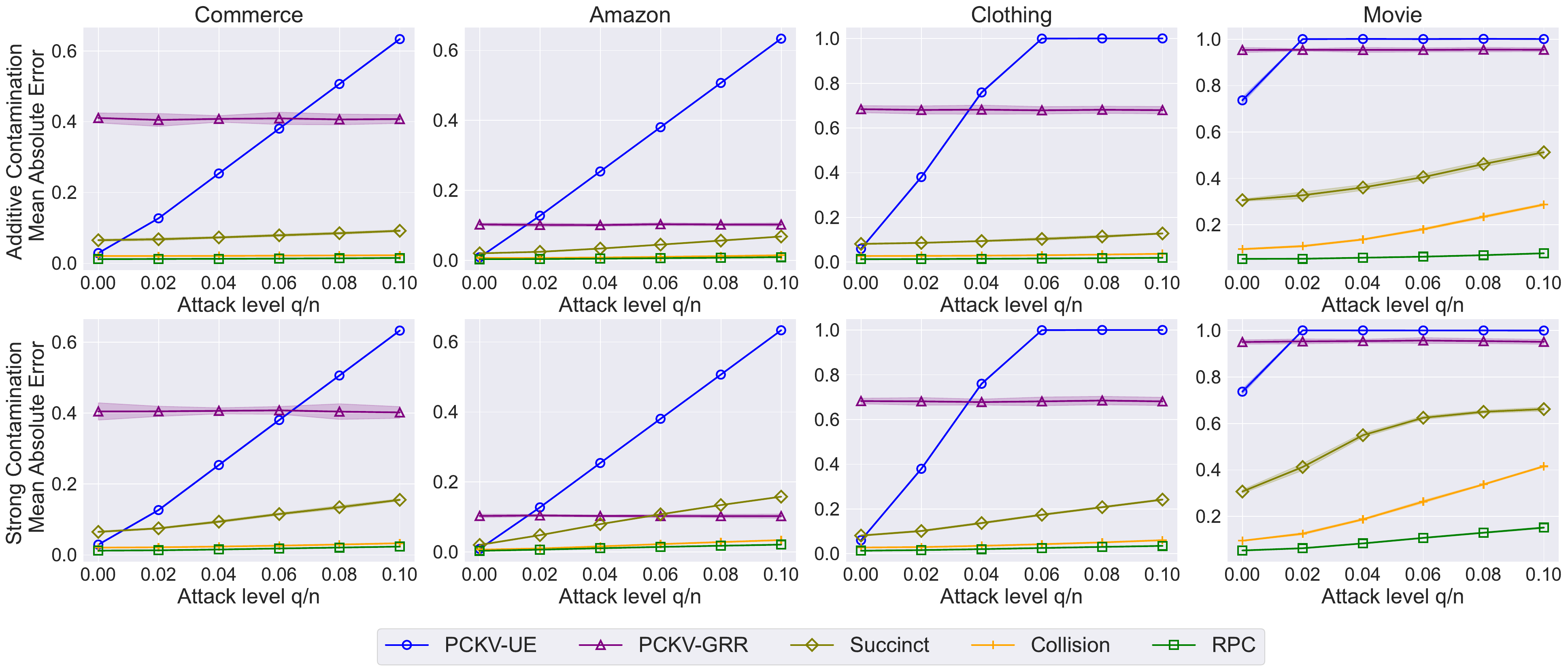}
	\caption{The mean absolute error of mean estimation in untrusted environment for four real datasets.}\label{fig:meanreal_attack}
	\vspace{3mm}
\end{figure*}
\mypara{Frequency estimation} The results are shown in \autoref{fig:freq}. For both PCKV-UE and PCKV-GRR, the errors are relatively large. As discussed earlier, PCKV samples one item from $m$ items belonging to each user, resulting in unnecessary information loss. This inefficiency becomes more obvious as $m$ increases. The succinct mechanism \cite{zhou2022locally} performs slightly worse than PCKV-UE with small $m$, but better than PCKV-UE with large $m$. Our explanation is that on the one hand, succinct mechanism has a hard bias-variance tradeoff. On the other hand, it handles multi-item users in a better way instead of a simple subsampling, which is important with large $m$. Collision \cite{wang2023differentially} significantly improves the performance, as it handles multi-item users in a better way. Our new proposed RPC method achieves comparable performance with Collision in general. To be more precise, when $\epsilon<1$, our method is slightly better. When $\epsilon>2$, Collision is slightly better. Here we provide an intuitive explanation. The variance comes from two parts, including those caused by the LDP mechanism, and those caused by the randomness of predefined signal $\mathbf{S}_i$. In our method, the clipping operation with bias correction effectively controls the first part, i.e. the variance by the LDP mechanism, which is more important under small $\epsilon$. However, as $\epsilon$ increases the randomness of $\mathbf{S}_i$ remain the same, thus the second part does not decrease further with $\epsilon$. Therefore, with large $\epsilon$, the performance of our method is not exactly optimal. In the future, we would like to explore efficient estimators for large $\epsilon$.

\mypara{Mean estimation with values in $\{-1,1\}$} The results are shown in \autoref{fig:mean}. Now our method consistently performs the best among all methods. For Collision \cite{wang2023differentially}, the mean estimation problem is converted to frequency estimation over $2d$ expanded dimensions, which increases the estimation error. In contrast, our method does not expand the dimensionality, and thus avoids the increase of estimation error.

\mypara{Mean estimation with continuous numerical values} The results are shown in \autoref{fig:meanimb}. From the results, now our method significantly outperforms all existing methods, including Collision \cite{wang2023differentially}. The Collision method maps all values within $[-1,1]$ to either $1$ or $-1$, which introduces some additional estimation variance. In contrast, our method calculates the projected value $u_i$ directly. Recall that values are distributed according to $\mathcal{N}(0,\sigma^2)$ with $\sigma = 0.2$. \textcolor{black}{It can be easily shown that \autoref{ass:beta} is satisfied with $\beta = 0.3$, thus we use $\beta=0.3$ in our experiment.} According to the analysis in \autoref{sec:general-trusted}, by adjusting the clipping threshold $R$ to control the bias, our RPC method can now perform better than other solutions based on randomizing values to $-1$ or $1$.

\mypara{Mean estimation using real datasets} Finally, we test the performance of our method using four real datasets: Commerce, Amazon, Clothing and Movie. The details of these datasets are shown in \autoref{sec:setup}. From \autoref{fig:meanreal}, it can be observed that our method consistently outperforms existing approaches.

\subsection{Experiments under Untrusted Environments}

Now we fix $\epsilon = 1$. We run experiments under both additive and strong contamination model. We let the number of attacked users $q$ vary from $0$ to $0.1n$.

\mypara{Frequency estimation} The results are shown in \autoref{fig:freq_attack}. It can be observed that PCKV-UE is not robust to poisoning attacks. For PCKV-GRR, the estimation is relatively large, but increases very slowly with $q$. Our explanation is that the output of generalized randomized response is a single scalar. Due to small output space, the effect of adversarial manipulation is limited, thus PCKV-GRR is actually more robust to attacks. Collision \cite{wang2023differentially} significantly improves over PCKV-UE and PCKV-GRR. However, with large $m$, the error still increases significantly with $q$. Compared with existing methods, our RPC method achieves both small estimation error at $q=0$, and slow growth with the increase $q$, indicating that our method is both accurate and more robust to poisoning attacks.

\mypara{Mean estimation using synthesized datasets} The results are shown in \autoref{fig:mean_attack}. The observations are similar to the frequency estimation task. These results show that compared with existing approaches, our new method is more robust to poisoning attacks.

\mypara{Mean estimation using real datasets} The results are shown in \autoref{fig:meanreal_attack}. It can be observed that with the increase of $q$, the error of PCKV-UE increases quickly to $1$, indicating that PCKV-UE is not robust to poisoning attacks. For Commerce, Amazon and Clothing, both Collision and our RPC method exhibit good robustness against poisoning attacks. Note that these three datasets have a relatively small $m$. For the movie dataset, which has a large $m$ ($m=100$), Collision is much more vulnerable to poisoning attacks, while our method performs relatively stable with $q$ increases.
\section{Related Works}\label{sec:related}

\subsection{Single-item Frequency Estimation}
Frequency estimation under LDP is a fundamental problem that received widespread attension. The earliest attempt can be traced back to Warner's randomized response \cite{warner1965randomized}, for the case with $d=2$. For larger $d$, a standard method is RAPPOR \cite{erlingsson2014rappor}. After that, \cite{wang2017locally} proposed Optimized Unary Encoding (OUE) and Optimized Local Hashing (OLH). \textcolor{black}{\cite{bassily2015local} proposed succinct histograms.} These methods achieve optimal error under some range of $\epsilon$. Several subsequent works achieve optimal error for all $\epsilon$. \cite{wang2016mutual} and \cite{ye2018optimal} proposed the subset selection method. \cite{acharya2019hadamard} proposed Hadamard response, which further reduces the communication complexity. Moreover, \cite{fang2025further} gives some further improvement to existing methods; \cite{li2020estimating,zhao2026consistent} analyzed distribution estimation for numerical data. There are also some works on heavy hitter estimation \cite{boneh2021lightweight,wang2019locally}, which is closely related to frequency estimation. 
%\vspace{-3mm}

\subsection{Vector Mean Estimation}

There are several early works on mean estimation over dense vectors \cite{duchi2013local,duchi2018minimax,nguyen2016collecting,bhowmick2018protection,asi2022optimal,asi2024fast,wang2019collecting}. These works do not consider sparse vectors. \cite{chen2022breaking} proposes a method for the $1$-sparse case with optimal error and succinct communication. The general $m$-sparse case is much more challenging. \cite{qin2016heavy} proposed sampling based methods for frequency estimation. Later on, PrivKVM \cite{ye2019privkv,ye2021privkvm} and PCKV \cite{gu2020pckv} are proposed to handle general numerical values. However, these methods work by sampling one item from $m$ items, which leads to information loss, and thus the performances are suboptimal.

In recent years, there are several new methods that use all items within a user efficiently. The Wheel method \cite{wang2018privset} was proposed for frequency estimation for $m$-sparse vectors. \cite{zhou2022locally} proposed Succinct mechanism for general sparse vectors. \cite{wang2023differentially} further proposed the Collision method to handle sparse vector aggregation, and the analysis is extended to shuffle model \cite{erlingsson2019amplification,feldman2022hiding,feldman2023stronger}. Recently, \cite{zhao2026sparse} significantly improved the performance of sparse frequency and mean estimation for cases $\epsilon \geq 1$, which is actually a modification of existing subset selection algorithms \cite{ye2018optimal,wang2016mutual}.

\subsection{Attacks and Defense on LDP Mechanisms}

It is well known that LDP mechanisms are vulnerable to poisoning attacks \cite{cheu2021manipulation,cao2021data,wu2022poisoning}. \cite{cao2021data} proposed three common attacks: Random Perturbation Attack (RPA), Random Item Attack (RIA), and Maximum Gain Attack (MGA). \cite{wu2022poisoning} extends the work to key-value data. \cite{li2023fine} proposed an attack that can distort both mean and variance estimation. \cite{tong2024data} studied attacks on frequent itemset mining.

For defense strategies, \cite{cao2021data} proposed several defenses against RPA, RIA and MGA \cite{huang2024ldpguard} proposed LDPGuard. However, these methods rely on the knowledge of the attack strategy. There are also some works based on post-processing, such as LDPRecover \cite{sun2024ldprecover} and Calibrate \cite{jia2019calibrate}. These defenses are lack of rigorous theoretical guarantee, and may still be circumvented by a carefully designed attack strategy. \cite{cheu2021manipulation} proposed a general scheme for robust LDP mechanisms for $\epsilon\leq 1$. \cite{zhao2025attack} extended the work to cases with $\epsilon>1$.
\section{Conclusion}\label{sec:conc}

Robustness against poisoning attacks is a prominent issue faced by LDP mechanisms. In this paper, our goal is to achieve robust estimation of the mean of sparse vectors. We propose the random projection with clipping method. The basic idea is to project each vector onto a randomly generated binary vector, and then clip the projected value. The clipping process induces bias. However, we provide a careful analysis, which shows that the bias can be corrected by multiplying a factor $\alpha$. As a result, we can use a small clipping threshold without while still ensuring unbiasedness. As a result, a weaker noise is enough for privacy protection. We establish an upper bound of the estimation error under poisoning attacks for the first time. Experiments also show that our method is significantly less vulnerable to attacks than existing methods.

While our method has achieved superior performance, currently we only focus on cases with $\epsilon = O(1)$. In the future, we would like to explore optimal robust estimation methods under $\epsilon$-LDP with large $\epsilon$.

\section*{Ethical Considerations}
This work does not present any ethical issues. The paper focuses on sparse estimation under local differential privacy. All experiments are conducted using either synthesized or public datasets, ensuring no exposure to personal information. Moreover, we strictly adhere to ethical guidelines throughout the research process, ensuring that no other ethical concerns arise.

\section*{Open Science}
To facilitate transparency and reproducibility, we provide the artifacts necessary to evaluate the results and claims made in thispaper, All artifacts are made available through anonymized links, and do not reveal author identities.

\mypara{Provided Artifacts} These artifacts associated with this submission include:
\begin{itemize}
    \item Source codes of our new proposed RPC method;

    \item Codes used for environment setting;

    \item Experimental results.
\end{itemize}

\mypara{Artifact Access}.  All code and documentation are available via an anonymous repository: 
\begin{itemize}
    \item https://anonymous.4open.science/r/robust\_sparse\_estimation-C765/
\end{itemize}

\bibliographystyle{ACM-Reference-Format}
\bibliography{ldp}	
\appendix
\section{Discussions of Randomizers}\label{sec:discuss}

Here we analyze the values of $c_\epsilon$ and $V_\epsilon$ of randomizers listed in \autoref{sec:randomizer}. Recall that $c_\epsilon$ and $V_\epsilon$ are defined in \eqref{eq:cepsdf} and \eqref{eq:vepsdf}, respectively.

\mypara{Laplace mechanism \cite{dwork2006calibrating}} Recall that
$Y=u+W$, with $W$ following Laplace distribution $\Lap(2/\epsilon)$. $W$ is unbounded, thus from \eqref{eq:cepsdf}, $c_\epsilon=\infty$.

Note that the probability density function (pdf) of $\Lap(\lambda)$ is $f(u)=e^{-|u|/\lambda}/(2\lambda)$. It can be shown that the variance is $2\lambda^2$. Therefore
\begin{eqnarray}
	\mathbb{E}[W] = 2\left(\frac{2}{\epsilon}\right)^2 = \frac{8}{\epsilon^2}.
\end{eqnarray}
From \eqref{eq:vepsdf}, $V_\epsilon = 8/\epsilon^2$.

\mypara{Piecewise mechanism (PM) \cite{wang2019collecting}} 

For PM, by \eqref{eq:pmpdf},
\begin{eqnarray}
	c_\epsilon = b = \frac{e^\frac{\epsilon}{2}+1}{e^\frac{\epsilon}{2} - 1},
\end{eqnarray}
and 
\begin{eqnarray}
V_\epsilon = \frac{4e^\frac{\epsilon}{2}}{3(e^\frac{\epsilon}{2} - 1)^2}. 	
\end{eqnarray}

\mypara{Duchi et al.\cite{duchi2013local}} 
By \eqref{eq:duchi},
\begin{eqnarray}
	c_\epsilon = \frac{e^\epsilon+1}{e^\epsilon - 1},
\end{eqnarray}
and
\begin{eqnarray}
	V_\epsilon = \left(\frac{e^\epsilon+1}{e^\epsilon - 1}\right)^2.
\end{eqnarray}

\section{Sampling-based Estimators}\label{sec:sampling}

The first type of existing solutions is based on sampling one item from all items belonging to a user.

\mypara{PrivKVM \cite{ye2019privkv,ye2021privkvm}} This method runs in multiple iterations. In the first iteration, each user randomly samples $j\in [d]$. If $x_i(j)\neq 0$ (which means that user $i$ contains key $j$), then it uploads a key-value pair $(1, x_i(j))$. Otherwise, it uploads $(0, \tilde{v})$, in which $\tilde{v}=0$ in the first iteration. The key is then perturbed with budget $\epsilon /2$, and the value is perturbed with budget $\epsilon / (2c)$, in which $c$ is the number of iterations. In the remaining iterations, each user perturbs its data with similar way, but now $\tilde{v}=m_k$, which is the estimated mean value from previous rounds. The value is then perturbed with budget $\epsilon/(2c)$. After $c$ rounds of interaction between users and the server, the whole privacy consumption is $\epsilon$. Moreover, \cite{ye2019privkv} shows that with a proper aggregation, the final estimation is approximately unbiased.

\mypara{Private Correlated Key-Value (PCKV) \cite{gu2020pckv}} For each user $i$, the encoding step outputs $(k_i, v_i)$, in which $k_i$ is randomly drawn from indices of nonzero components of $\mathbf{x}_i$, and $v_i=x_i(k_i)$. For the perturbation step, there are various randomizers. Here we present PCKV-UE as an example. The user $i$ generates $\mathbf{y}_i\in \{-1,0,1\}^d$ as follows. $y_i(k_i)$ is generated by
\begin{eqnarray}
	y_i(k_i) = \left\{
	\begin{array}{ccc}
		v_i &\text{with probability} & ap\\
		-v_i &\text{with probability} & a(1-p)\\
		0 &\text{with probability} & 1-a.		
	\end{array}
	\right.
\end{eqnarray}
For $j\neq k_i$,
\begin{eqnarray}
	y_i(j) = \left\{
	\begin{array}{ccc}
		1 &\text{with probability} & b/2\\
		-1 &\text{with probability} & b/2\\
		0 &\text{with probability} & 1-b.		
	\end{array}
	\right.
\end{eqnarray}
In the aggregation step,
\begin{eqnarray}
	\hat{\mu}_c(j) = \frac{m}{na(2p-1)} \sum_{i=1}^n y_i(j).
	\label{eq:pckv}
\end{eqnarray}
The values of $a,b$ need to satisfy
$a(1-b)/(b(1-a)) = e^{\epsilon_1}$,
and
$p = e^{\epsilon_2}/(e^{\epsilon_2} + 1)$.
Then the whole algorithm is $\epsilon$-LDP with
	$\epsilon = \max\left\{\epsilon_2, \epsilon_1 + \ln 2/(1+e^{-\epsilon_2}) \right\}$.

For all methods based on sampling, each user only sends a small fraction of information to the server. The bound of $\ell_1$ error is $O(\sqrt{d^3/(n\epsilon^2)})$ for PrivKVM, and $O(dm/\sqrt{n\epsilon^2})$ for PCKV\footnote{The derivation of these bounds are omitted due to limited space.}, while the theoretical lower bound is $O(d\sqrt{m/(n\epsilon^2)})$, indicating a room for improvement. More importantly, as will be shown later, these methods are vulnerable to poisoning attacks.

\section{Proof of \autoref{lem:unbiased}}\label{sec:unbiased}
Recall that $Q$ denotes arbitrary randomizer described in \autoref{sec:randomizer}. Laplace mechanism, PM and Duchi et al.'s method \cite{duchi2013local} are all unbiased, thus $\mathbb{E}[Y_i|u_i] = u_i$. Therefore
\begin{eqnarray}
	\mathbb{E}[\hat{\mu}_c] &=& \mathbb{E}\left[\frac{1}{n}\sum_{i=1}^n u_i\mathbf{S}_i\right]\nonumber\\
	&=& \mathbb{E}\left[\frac{1}{n}\sum_{i=1}^n \langle \mathbf{x}_i, \mathbf{S}_i\rangle \mathbf{S}_i\right].
\end{eqnarray}
Hence
\begin{eqnarray}
	\mathbb{E}[\hat{\mu}_c(j)] = \mathbb{E}\left[\frac{1}{n}\sum_{i=1}^n \sum_{l=1}^d x_i(l)S_i(l)S_i(j)\right].
\end{eqnarray}
Note that $\mathbf{S}_i$ is a random vector that take values uniformly from $\{-1,1\}^d$, thus
\begin{eqnarray}
	\mathbb{E}[S_i(l)S_i(j)] = \left\{
	\begin{array}{ccc}
		1 &\text{if} & l=j\\
		0 &\text{if} & l\neq j.
	\end{array}
	\right.
\end{eqnarray}
Therefore 
\begin{eqnarray}
	\mathbb{E}[\hat{\mu}_c(j)] = \frac{1}{n}\sum_{i=1}^n x_i(j) = \mu(j),
\end{eqnarray}
i.e.
\begin{eqnarray}
	\mathbb{E}[\hat{\mu}_c] = \mu.
\end{eqnarray}
The proof of \autoref{lem:unbiased} is complete.
\section{Proof of \autoref{lem:clip}}\label{sec:clip}
Define
\begin{eqnarray}
	\hat{\mu}_0 = \frac{\alpha}{n}\sum_{i=1}^n u_i \mathbf{S}_i.
	\label{eq:mu0}
\end{eqnarray}
From \eqref{eq:mu0},
\begin{eqnarray}
	&&\frac{1}{n} \sum_{i=1}^n u_i\mathbf{S}_i - \mu(j) \nonumber\\
	&=& \frac{1}{n}\sum_{i=1}^n u_iS_i(j) - \frac{1}{n}\sum_{i=1}^n x_i(j)\nonumber\\
	&=& \frac{1}{n}\sum_{i=1}^n \left[\Clip(\langle \mathbf{x}_i, \mathbf{S}_i\rangle, R)S_i(j) - x_i(j)\right] \nonumber\\
	&=& \frac{1}{n}\sum_{i=1}^n \left[\Clip(\langle \mathbf{x}_i,\mathbf{S}_i\rangle s_i(j), R)-x_i(j)\right]\nonumber\\
	&=& \frac{1}{n} \sum_{i=1}^n \left[\Clip(x_i(j)+U_{ij}, R) - x_i(j)\right],
\end{eqnarray}
in which
\begin{eqnarray}
	U_{ij} := \sum_{l\neq j} x_i(l)s_i(l)s_i(j).
\end{eqnarray}
If $x_i(j) = 0$, then 
\begin{eqnarray}
	\mathbb{E}[\Clip(x_i(j)+U_{ij}, R) - x_i(j)] = 0.
\end{eqnarray}
If $x_i(j)\in \{-1,1\}$, then among all indices $l$ with $l\neq j$, there are $m-1$ remaining nonzero values. Therefore, $U_{ij}$ is the sum of $m-1$ Rademacher random variables\footnote{Rademacher random variable has $1/2$ probability at $-1$ and $1$, respectively.}. 

If $x_i(j)=-1$, then
%Without loss of generality, assume that $x_i(j)=-1$, and the case with $x_i(j)=1$ can be analyzed similarly. 
\begin{eqnarray}
	&&\mathbb{E}\left[\Clip(x_i(j)+U_{ij}, R) - x_i(j)\right]\nonumber\\
	&& = \mathbb{E}[\Clip(U_{ij}, -R+1, R+1)],
\end{eqnarray}
in which $\Clip(u,a,b) = \max\{a, \min\{u, b\}\}$. 

Similar results can be derived for $x_i(j)=1$:
\begin{eqnarray}
	&&\mathbb{E}\left[\Clip(x_i(j)+U_{ij}, R) - x_i(j)\right]\nonumber\\
	&& = \mathbb{E}[\Clip(U_{ij}, -R-1, R-1)].
\end{eqnarray}
Therefore
\begin{eqnarray}
	&&\mathbb{E}\left[\frac{1}{n}\sum_{i=1}^n Y_i\mathbf{S}_i(j) \right] - \mu(j) \nonumber\\
	&=& \frac{1}{n} \sum_{i=1}^n \mathbf{1}(x_i(j) = 1)\mathbb{E}\left[\Clip(U_{ij}, -R-1, R-1)\right] \nonumber\\
	&&+ \frac{1}{n} \sum_{i=1}^n \mathbf{1}(x_i(j) = -1)\mathbb{E}\left[\Clip(U_{ij}, -R+1, R+1)\right]\nonumber\\
	&=&-\frac{1}{n} \sum_{i=1}^n \mathbf{1}(x_i(j) = 1)\mathbb{E}\left[\Clip(U_{ij}, -R+1, R+1)\right] \nonumber\\
	&&+ \frac{1}{n} \sum_{i=1}^n \mathbf{1}(x_i(j) = -1)\mathbb{E}\left[\Clip(U_{ij}, -R+1, R+1)\right]\nonumber\\
	&=& -\mu(j) \mathbb{E}\left[\Clip(U_{ij}, -R+1, R+1)\right].
\end{eqnarray}
Reorganize the above equation, we can get
\begin{eqnarray}
	\mu(j) = \mathbb{E}\left[\frac{1}{1-\mathbb{E}[\Clip(U_{ij}, -R+1,R+1)]} \frac{1}{n}\sum_{i=1}^n Y_i\mathbf{S}_i(j)\right].\hspace{-1cm}\nonumber\\
\end{eqnarray}
Hence, with 
\begin{eqnarray}
	\alpha = \frac{1}{1-\mathbb{E}[\Clip(U_{ij}, -R+1,R+1)]},
	\label{eq:alphae}
\end{eqnarray}
$\hat{\mu}_0$ defined in \eqref{eq:mu0} is unbiased. 

Finally, we discuss the exact calculation of the value of $\alpha$. Note that $\mathbb{E}[\Clip(U_{ij}, -R+1,R+1)]$ can be calculated analytically. Recall that $U_{ij}$ can be viewed as the sum of $m-1$ Rademacher random variables, thus
\begin{eqnarray}
	&&\mathbb{E}[\Clip(U_{ij}, -R+1,R+1)] \nonumber\\
	&&= \sum_{l=0}^{m-1} \binom{m-1}{l} 2^{-(m-1)} C_{l,m, R},
\end{eqnarray}
in which $C_{l,m,R} = \Clip(2l+1-m, -R+1, R+1)$. The proof of \autoref{lem:clip} is complete.

\section{Estimation Error in Trusted Environments}\label{sec:rpc}
This section proves \autoref{thm:rpc} and \autoref{cor}.

\subsection{Proof of \autoref{thm:rpc}}
Recall the definition of $\hat{\mu}_0$ in \eqref{eq:mu0}. The $\ell_2$ error can be decomposed as follows.
\begin{eqnarray}
	\mathbb{E}\left[\norm{\hat{\mu}_c-\mu}_2^2\right] &=& \mathbb{E}\left[\norm{\hat{\mu}_c - \hat{\mu}_0}_2^2\right] + \mathbb{E}\left[\norm{\hat{\mu}_0-\mu}_2^2\right] \nonumber\\
	&&+ \mathbb{E}\left[\langle \hat{\mu}_c-\hat{\mu}_0, \hat{\mu}_0-\mu\rangle\right].
\end{eqnarray}
Denote $y_{1:n} = (y_1,\ldots, y_n)$ for convenience. Note that
\begin{eqnarray}
	&&\mathbb{E}\left[\langle \hat{\mu}_c-\hat{\mu}_0,\hat{\mu}_0-\mu\rangle |y_{1:n}\right] \nonumber\\
	&&= \mathbb{E}\left[\left\langle \frac{\alpha}{n}\sum_{i=1}^n (Y_i-y_i)\mathbf{S}_i, \hat{\mu}_0-\mu\right\rangle|y_{1:n}\right] = 0.  \nonumber\\
\end{eqnarray}
Therefore
\begin{eqnarray}
	\mathbb{E}\left[\norm{\hat{\mu}_c-\mu}_2^2\right] = \mathbb{E}\left[\norm{\hat{\mu}_c - \hat{\mu}_0}_2^2\right] + \mathbb{E}\left[\norm{\hat{\mu}_0-\mu}_2^2\right].
	\label{eq:l2dec}
\end{eqnarray}
We bound two terms in \eqref{eq:l2dec} separately. 

\textbf{Bound of $\mathbb{E}\left[\norm{\hat{\mu}_c-\hat{\mu}_0}_2^2\right]$.}
\begin{eqnarray}
	\mathbb{E}\left[\norm{\hat{\mu}_c-\hat{\mu}_0}_2^2\right]  &=& \mathbb{E}\left[\norm{\frac{\alpha}{n}\sum_{i=1}^n (Y_i-y_i)\mathbf{S}_i}_2^2\right]\nonumber\\
	&=& \frac{\alpha^2}{n^2} \sum_{i=1}^n \mathbb{E}\left[\norm{(Y_i-y_i)\mathbf{S}_i}_2^2\right]\nonumber\\
	&=& \frac{\alpha^2}{n^2} \sum_{i=1}^n \sum_{j=1}^d \Var[Y_i|y_i]\nonumber\\
	&=& \frac{\alpha^2}{n}V_\epsilon R^2 d,
	\label{eq:l21}
\end{eqnarray}
in which $V_\epsilon$ depends on the LDP protocol, which has been discussed in \autoref{sec:randomizer}. Under the strong privacy regime ($\epsilon = O(1)$), for Laplace mechanism, PM and Duchi et al.'s method \cite{duchi2013local}, $V_\epsilon=O(1/\epsilon^2)$.

\textbf{Bound of $\mathbb{E}[\norm{\hat{\mu}_0-\mu}_2^2]$.} From \eqref{eq:mu0},
\begin{eqnarray}
	\hat{\mu}_0(j) = \frac{\alpha}{n} \sum_{i=1}^n \Clip(\langle \mathbf{x}_i, \mathbf{S}_i\rangle, R) S_i(j).
\end{eqnarray}
The variance of each term is bounded by
\begin{eqnarray}
	&&\Var[\Clip(\langle \mathbf{x}_i, \mathbf{S}_i\rangle, R)s_i(j)] \nonumber\\
	&=& \Var[\Clip(\langle \mathbf{x}_i, \mathbf{S}_i\rangle s_i(j), R)]\nonumber\\
	&\leq & \Var\left[\langle \mathbf{x}_i, \mathbf{S}_i\rangle S_i(j)\right]\nonumber\\
	&=& \Var\left[\sum_{l=1}^d x_i(l) s_i(l)s_i(j)\right]\nonumber\\
	&=& \Var\left[\sum_{l\neq j} x_i(l)s_i(l)s_i(j) \right]\nonumber\\
	&=&\sum_{l \neq j} x_i^2(l)\nonumber\\
	&\leq & m.
\end{eqnarray}
Therefore
\begin{eqnarray}
	\Var[\hat{\mu}_0(j)]\leq \frac{\alpha^2 m}{n}.
\end{eqnarray}
It has been proved in \autoref{lem:unbiased} that $\hat{\mu}_0$ is unbiased. Thus
\begin{eqnarray}
	\mathbb{E}\left[\norm{\hat{\mu}_0-\mu}_2^2\right] = \sum_{j=1}^d \Var[\hat{\mu}_0(j)] = \frac{1}{n}\alpha^2 m d.
	\label{eq:l22}
\end{eqnarray}
Combine \eqref{eq:l2dec}, \eqref{eq:l21} and \eqref{eq:l22},
\begin{eqnarray}
	\mathbb{E}\left[\norm{\hat{\mu}_c-\mu}_2^2\right] \leq \frac{\alpha^2}{n}(V_\epsilon R^2 + m)d.
\end{eqnarray}
By Cauchy-Schwartz inequality,
\begin{eqnarray}
	\mathbb{E}\left[\norm{\hat{\mu}_c-\mu}_1\right] &\leq& \sqrt{d \mathbb{E}\left[\norm{\hat{\mu}_c-\mu}_2^2\right]}\nonumber\\
	&\leq & d\alpha \sqrt{\frac{V_\epsilon R^2+m}{n}}
\end{eqnarray}
The proof of \autoref{thm:rpc} is complete.

\subsection{Proof of \autoref{cor}}
Recall that $U_{ij}$ is the sum of $m-1$ independent Rademacher random variables. For any $R$ taking integer values,
\begin{eqnarray}
	&&\mathbb{E}[\Clip(U_{ij}, -R+1, R+1)]\nonumber\\
	&=& \mathbb{E}[U_{ij} \mathbf{1}(-R+1\leq U_{ij}\leq R+1)]\nonumber\\
	&& - (R-1)\text{P}(U_{ij}<-R+1)\nonumber\\
	&&+(R+1)\text{P}(U_{ij} > R+1)\nonumber\\
	&\overset{(a)}{=} & \mathbb{E}[U_{ij} \mathbf{1}(R-1<U_{ij}\leq R+1)]\nonumber\\
	&& -(R-1) \text{P}(U_{ij}<-R-1)\nonumber\\
	&& -(R-1)\text{P}(-R-1\leq U_{ij} <-R+1)\nonumber\\
	&&+(R+1) \text{P}(U_{ij} > R+1)\nonumber\\
	&\overset{(b)}{=}& \mathbb{E}[U_{ij} \mathbf{1}(R-1<U_{ij}\leq R+1)] + 2\text{P}(U_{ij} > R+1) \nonumber\\
	&&- (R-1)\text{P}(R-1<U_{ij} \leq R+1)\nonumber\\
	&\overset{(c)}{\leq}& 2\text{P}(R-1<U_{ij} \leq R+1)+2\text{P}(U_{ij} > R+1)\nonumber\\
	&=& 2\text{P}(U_{ij} > R-1) \nonumber\\
	&\overset{(d)}{=}& 2\text{P}(U_{ij} \geq R)\nonumber\\
	&\overset{(e)}{\leq} & 2e^{-\frac{R^2}{2m}}.
\end{eqnarray}
In the above steps, (a) holds because by the symmetry of $U_{ij}$,
\begin{eqnarray}
	\mathbb{E}[U_{ij} \mathbf{1}(-R+1\leq U_{ij}\leq R-1)] = 0.
\end{eqnarray}
(b) also comes from the symmetry of $U_{ij}$, i.e.
\begin{eqnarray}
	\text{P}(-R-1\leq U_{ij} <-R+1)=\text{P}(R-1<U_{ij}\leq R+1).
\end{eqnarray}
(c) holds since when $R-1<U_{ij}\leq R+1$, $U_ij-(R-1)\in [0,2]$. (d) holds because $R$ is an integer, and $U_{ij}$ only takes integer values, thus $U_{ij}>R-1$ is equivalent to $U_{ij}>R$. (e) comes from Hoeffding's inequality \cite{boucheron2003concentration}.

With $R=\lceil \sqrt{2m}\rceil$, 
\begin{eqnarray}
	\mathbb{E}[\Clip(U_{ij}, -R+1, R+1)] \leq 2e^{-1}.
\end{eqnarray}
By \eqref{eq:alphae},
\begin{eqnarray}
	\alpha \leq \frac{1}{1-2e^{-1}}.
\end{eqnarray}
The proof of \autoref{cor} is complete.

\section{Proof of \autoref{thm:attack}}\label{sec:attack}

\color{black}In this section, we first analyze under the strong contamination model, and then show that the $\sqrt{\ln n}$ factor can be removed under the additive contamination model.\color{black}

\mypara{Under strong contamination model}
From triangle inequality,
\begin{eqnarray}
	\norm{\hat{\mu}-\mu}_1\leq \norm{\hat{\mu}-\hat{\mu}_c}_1 + \norm{\hat{\mu}_c-\mu}_1.
    \label{eq:decomp}
\end{eqnarray}
The second term has been analyzed in \autoref{thm:rpc}. It remains to bound the first term. Recall that $\hat{\mu}_c = (\alpha/n) \sum_{i=1}^n Y_i\mathbf{S}_i$. From \eqref{eq:rpc},
\begin{eqnarray}
	&&\mathbb{E}\left[\max_{Z_{1:n}\in \mathcal{Z}} \norm{\hat{\mu}-\hat{\mu}_c}_1\right]\nonumber\\
	&=& \mathbb{E}\left[\max_{Z_{1:n}\in \mathcal{Z}} \norm{\frac{\alpha}{n}\sum_{i=1}^n (Z_i-Y_i)\mathbf{S}_i}_1\right]\nonumber\\
	&\overset{(a)}{=}& \mathbb{E}\left[\max_{Z_{1:n}\in \mathcal{Z}^*} \norm{\frac{\alpha}{n}\sum_{i=1}^n (Z_i-Y_i)\mathbf{S}_i}_1\right]\nonumber\\
	&\overset{(b)}{=}& \mathbb{E}\left[\max_\mathcal{C}\max_{Z_C} \norm{\frac{\alpha}{n}\sum_{i=1}^n (Z_i-Y_i)\mathbf{S}_i}_1\right] \nonumber\\
	&\leq & \mathbb{E}\left[\max_\mathcal{C}\max_{Z_C} \norm{\frac{\alpha}{n}\sum_{i\in \mathcal{C}} Z_i\mathbf{S}_i}_1\right] + \mathbb{E} \left[\norm{\frac{\alpha}{n} \sum_{i\in \mathcal{C}} Y_i\mathbf{S}_i}_1\right]\nonumber\\
	&\leq & \frac{2\alpha}{n}\mathbb{E}\left[\max_\mathcal{C}\max_{Z_C} \norm{\sum_{i\in \mathcal{C}} Z_i\mathbf{S}_i}_1\right]\nonumber\\
	&=& \frac{2\alpha}{n} \mathbb{E}\left[\max_\mathcal{C}\max_{Z_C} \max_{\mathbf{u}\in \{-1,1\}^d} \sum_{j=1}^d u_j\sum_{i\in \mathcal{C}} Z_iS_i(j)\right].
	\label{eq:attack1}
\end{eqnarray}
In (a),
\begin{eqnarray}
	\mathcal{Z}^* = \{(Z_1,\ldots, Z_n)|Z_i\in \{Y_i, c_\epsilon R, -c_\epsilon R \},
	|\{i|Z_i\neq Y_i\}|\leq q\}.\hspace{-10mm}\nonumber\\
\end{eqnarray}
(a) holds since $\ell_1$ error is convex, thus the maximum $\ell_1$ error is reached at the boundary.

In (b), $\max_\mathcal{C}$ takes the maximum over all subsets of $[n]$ with size $q$. $\max_{Z_C}$ takes maximum over $\{-c_\epsilon R, c_\epsilon R\}^q$. 

Fix $\mathbf{u}$ and $Z_i$, $i=1,\ldots, n$, then $\sum_{j=1}^d u_j\sum_{i\in \mathcal{C}}Z_iS_i(j)$ is the sum of $qd$ random variables uniformly distributed in $\{-1,1\}$. By Hoeffding's inequality,
\begin{eqnarray}
	\text{P}\left(\sum_{j=1}^d u_j\sum_{i\in \mathcal{C}}Z_iS_i(j) > t\right) \leq  \exp\left(-\frac{t^2}{2qdc_\epsilon^2 R^2}\right).
\end{eqnarray}
Now we take union bound. There are $\binom{n}{q}$ possible selection of set $\mathcal{C}$. $Z_C$ can take $2^q$ values, and $\mathbf{u}$ can take $2^d$ values. Therefore
\begin{eqnarray}
	&&\text{P}\left(\max_\mathcal{C} \max_{Z_C} \max_{\mathbf{u}\in \{-1,1\}^d} \sum_{j=1}^d u_j\sum_{i\in \mathcal{C}}Z_iS_i(j)>t\right) \nonumber\\
	&&\leq \binom{n}{q} 2^{d+q} \exp\left(-\frac{t^2}{2qdc_\epsilon^2 R^2}\right).
    \label{eq:plarge}
\end{eqnarray}
To bound the expectation, we use the following lemma.
\begin{lem}\label{lem:expect}
	For a random variable $X$, if $X\geq 0$ almost surely, and 
	\begin{eqnarray}
		\text{P}(X>t)\leq ae^{-bt^2},
	\end{eqnarray}
	then
	\begin{eqnarray}
		\mathbb{E}[X]\leq \sqrt{\frac{\ln a}{b}}+\sqrt{\frac{\pi}{b}}.
	\end{eqnarray}
\end{lem}
\color{black}
\begin{proof}
    Note that with $t_c=\sqrt{(\ln a)/b}$, $ae^{-bt_c^2} = 1$. Then
\begin{eqnarray}
E[X]&=&\int_0^\infty P(X>t) dt = \int_0^{t_c} P(X>t) dt+\int_{t_c}^\infty P(X>t) dt \nonumber \\
&\leq & t_c+\int_{t_c}^\infty  ae^{-bt^2} dt \nonumber\\
&=& t_c+\int_{t_c}^\infty  e^{-b(t^2-t_c^2)}dt\nonumber\\
&\leq & t_c +\int_{t_c}^\infty  e^{-b(t-t_c)^2}dt\nonumber\\
&=&t_c+(1/\sqrt{b})\int_0^\infty e^{-u^2}du\nonumber\\
&\leq &  t_c + \sqrt{\pi/b}.
\end{eqnarray}
\end{proof}
\color{black}
Based on \autoref{lem:expect}, let
\begin{eqnarray}
	a=\binom{n}{q} 2^{d+q}, b=\frac{1}{2qdc_\epsilon^2 R^2},
\end{eqnarray}
then \textcolor{black}{by \eqref{eq:plarge},}
\begin{eqnarray}
	&&\mathbb{E}\left[\max_\mathcal{C} \max_{Z_C} \max_{\mathbf{u}\in \{-1,1\}^d} \sum_{j=1}^d u_j\sum_{i\in \mathcal{C}}Z_iS_i(j) \right]\nonumber\\
	&& \lesssim c_\epsilon R\sqrt{qd}\sqrt{q\ln n+d},
	\label{eq:attack2}
\end{eqnarray}
\color{black}in which $\lesssim$ is used because some constant factors are omitted. According to \eqref{eq:attack1}, \color{black}
\begin{eqnarray}
	\mathbb{E}\left[\max_{Z_{1:n}\in \mathcal{Z}} \norm{\hat{\mu}-\hat{\mu}_c}_1\right]\lesssim R\left(\frac{\alpha q\sqrt{d\ln n}}{n}+\frac{\alpha d\sqrt{q}c_\epsilon }{n}\right).
	\label{eq:corrupt}
\end{eqnarray}
Recall \autoref{thm:rpc}, which shows that $\mathbb{E}\left[\norm{\hat{\mu}_c-\mu}_1\right] \lesssim d\alpha R\sqrt{V_\epsilon/n}$, the second term in \eqref{eq:corrupt} does not dominate. Hence
\color{black}
\begin{eqnarray}
	\mathbb{E}\left[\max_{Z_{1:n}\in \mathcal{Z}} \norm{\hat{\mu}-\mu}_1\right] \lesssim d\alpha R\sqrt{\frac{V_\epsilon}{n}}+c_\epsilon \frac{\alpha qR\sqrt{d\ln n}}{n}.
\end{eqnarray}
With $R=\sqrt{m}$, $\alpha$ is upper bounded by a constant, thus for $\epsilon \leq 1$,
\begin{eqnarray}
	\mathbb{E}\left[\max_{Z_{1:n}\in \mathcal{Z}} \norm{\hat{\mu}-\mu}_1\right]\lesssim d\sqrt{\frac{V_\epsilon m}{n\epsilon^2}}+c_\epsilon \frac{q\sqrt{dm\ln n}}{n\epsilon}.
\end{eqnarray}
\color{black}

\color{black}
\mypara{Under additive contamination model} In this case, we no longer need to take maximum over $\mathcal{C}$ at the right hand side of \eqref{eq:attack1}, since the adversary can not select the corruption set $\mathcal{C}$ arbitrarily. Instead, $\mathcal{C}$ is randomly selected from all samples. Following the proof regarding strong contamination model, it can be shown that 

\begin{eqnarray}
	\text{P}\left( \max_{Z_C} \max_{\mathbf{u}\in \{-1,1\}^d} \sum_{j=1}^d u_j\sum_{i\in \mathcal{C}}Z_iS_i(j)>t\right) \leq  2^{d+q} \exp\left(-\frac{t^2}{2qdc_\epsilon^2 R^2}\right).
    \label{eq:plarge2}
\end{eqnarray}
Compared with \eqref{eq:plarge}, now we have removed the $\binom{n}{q}$ factor. The remaining steps almost remain the same, thus we omit them for brevity. As a result,
\begin{eqnarray}
	\mathbb{E}\left[\max_{Z_{1:n}\in \mathcal{Z}} \norm{\hat{\mu}-\mu}_1\right]\lesssim d\sqrt{\frac{m}{n\epsilon^2}}+\frac{q\sqrt{dm}}{n\epsilon}.
\end{eqnarray}

The proof is complete.

\section{Analysis of the Robustness of Collision Mechanism}\label{sec:colanalysis}

In this section, we analyze the performance of the Collision mechanism \cite{wang2023differentially} under attacks. We still decompose the estimation error using \eqref{eq:decomp}. Then we focus on $\norm{\hat{\mu}-\hat{\mu}_c}$. We first analyze under additive contamination model. Similar to the analysis in \autoref{sec:attack}, Under strong contamination model, there is an additional $\sqrt{\ln n}$ factor.

\begin{eqnarray}
    \norm{\hat{\mu}-\hat{\mu}_c}_1 &=& \sum_{j=1}^d |\hat{\mu}(j)-\hat{\mu}_c(j)|\nonumber\\
    &=& \max_{\mathbf{u}\in \{-1,1\}^d} \sum_{j=1}^d u(j) (\hat{\mu}(j)-\hat{\mu}_c(j)).
\end{eqnarray}

Recall \eqref{eq:muccol}, 

\begin{eqnarray}
    &&\max_{Z_C} \norm{\hat{\mu}-\hat{\mu}_c}_1 \nonumber\\
    &=& \frac{1}{n} \frac{1}{e^\epsilon/\Omega - 1/t} \max_{\mathbf{u}\in \{-1,1\}^d}\max_{Z_C}\sum_{j=1}^d \sum_{i\in C}[\mathbf{1}(Z_i=H_i(j))\nonumber\\
    &&-\mathbf{1}(Y_i=H_i(j)]\nonumber\\
    &\leq & \frac{2}{n} \frac{1}{e^\epsilon/\Omega - 1/t}\max_{\mathbf{u}\in \{-1,1\}^d}\max_{Z_C} \sum_{j=1}^d \sum_{i\in C} \left(\mathbf{1}(Z_i=H_i(j))-\frac{1}{t}\right).\nonumber\\
\end{eqnarray}
By Hoeffding's inequality,
\begin{eqnarray}
    \text{P}\left(\sum_{j=1}^d \sum_{i\in C} \left(\mathbf{1}(Z_i=H_i(j))-\frac{1}{t}\right)>u\right) \leq e^{-\frac{2u^2}{qd}}.
\end{eqnarray}
Taking maximum over $\mathbf{u}$ and $Z_C$, we have
\begin{eqnarray}&&\text{P}\left(\max_{\mathbf{u}\in \{-1,1\}^d} \max_{Z_C}\sum_{j=1}^d \sum_{i\in C} \left(\mathbf{1}(Z_i=H_i(j))-\frac{1}{t}\right)>u\right)\nonumber\\
    &&\leq 2^{d+q} e^{-\frac{2u^2}{qd}}.
\end{eqnarray}

Therefore, from Lemma \ref{lem:expect},
\begin{eqnarray}
 \mathbb{E}\left[\max_{\mathbf{u}\in \{-1,1\}^d} \max_{Z_C}\sum_{j=1}^d \sum_{i\in C} \left(\mathbf{1}(Z_i=H_i(j))-\frac{1}{t}\right)\right] \lesssim d\sqrt{q}+q\sqrt{d}.
\end{eqnarray}

Recall that $\Omega = me^\epsilon+t-m$. For $\epsilon=O(1)$,
\begin{eqnarray}
    \max_{Z_C} \norm{\hat{\mu}-\hat{\mu}_c}_1\lesssim \frac{t}{n\epsilon} (d\sqrt{q}+q\sqrt{d}).
\end{eqnarray}
According to \cite{wang2023differentially}, the optimal value of $t$ is $t^*=2m-1+me^\epsilon$. Therefore, with $t=t^*$, for $\epsilon=O(1)$,
\begin{eqnarray}
      \max_{Z_C} \norm{\hat{\mu}-\hat{\mu}_c}_1\lesssim \frac{m}{n\epsilon} (d\sqrt{q}+q\sqrt{d}).  
\end{eqnarray}

\color{black}

\section{Proof of \autoref{thm:trusted}}\label{sec:trusted}
Note that now $\hat{\mu}_c = \frac{1}{n}\sum_{i=1}^n Y_i \mathbf{S}_i$. Then
\begin{eqnarray}
	\mathbb{E}[\hat{\mu}_c] = \mathbb{E}\left[\frac{1}{n}\sum_{i=1}^n Y_i\mathbf{S}_i\right] = \mathbb{E}\left[\frac{1}{n}\sum_{i=1}^n u_i \mathbf{S}_i\right].
\end{eqnarray}
Hence
\begin{eqnarray}
	&&\mathbb{E}[\hat{\mu}_c(j)] - \mu(j)\nonumber\\
	&=& \frac{1}{n}\sum_{i=1}^n (\mathbb{E}[u_iS_i(j)] - x_i(j))\nonumber\\
	&=& \frac{1}{n}\sum_{i=1}^n \left(\mathbb{E}[\Clip(\langle \mathbf{x}_i, \mathbf{S}_i\rangle, R) S_i(j)] - x_i(j)\right)\nonumber\\
	&=&\frac{1}{n}\sum_{i=1}^n \left(\mathbb{E}[\Clip(\langle \mathbf{x}_i, \mathbf{S}_i\rangle S_i(j), R)] - x_i(j)\right).
	\label{eq:biasexp}
\end{eqnarray}
Note that
\begin{eqnarray}
	\langle \mathbf{x}_i, \mathbf{S}_i\rangle S_i(j) &=& \sum_{l=1}^d x_i(l) S_i(l)S_i(j) \nonumber\\
	&=& x_i(j) +\sum_{l\neq j} x_i(l)S_i(l),
\end{eqnarray}
thus
\begin{eqnarray}
	&&\hspace{-7mm}\Clip(\langle \mathbf{x}_i, \mathbf{S}_i\rangle S_i(j), R) \nonumber\\
	&&\hspace{-7mm}= x_i(j) + \Clip \left(\sum_{l\neq j} x_i(l)S_i(l), -R-x_i(j), R-x_i(j)\right).\nonumber\\
	\label{eq:clip}
\end{eqnarray}
Define
\begin{eqnarray}
	U_{ij} := \sum_{l\neq j} x_i(l)S_i(l).
\end{eqnarray}
Put \eqref{eq:clip} into \eqref{eq:biasexp},
\begin{eqnarray}
	&&\mathbb{E}[\Clip(\langle \mathbf{x}_i, \mathbf{S}_i\rangle S_i(j), R)] - x_i(j)\nonumber\\
	&=& \mathbb{E}\left[\Clip\left(\sum_{l\neq j} x_i(l)S_i(l), -R-x_i(j), R-x_i(j)\right) \right]\nonumber\\
	&=& \mathbb{E}[U_{ij} \mathbf{1}(-R-x_i(j)\leq U\leq R-x_i(j))]\nonumber\\
	&&+ (R-x_i(j))\text{P}(U>R-x_i(j))	\nonumber\\
	&&-(R+x_i(j)) \text{P}(U<-R-x_i(j))\nonumber\\
	&=& \mathbb{E}[U\mathbf{1}(-R-x_i(j)\leq U< -R+x_i(j))] \nonumber\\
	&&+ (R-x_i(j))\text{P}(R-x_i(j)<R\leq R+x_i(j))\nonumber\\
	&&-2x_i(j) \text{P}(U>R+x_i(j)).
\end{eqnarray}
Since $|x_i(j)|\leq 1$, taking absolute value to the above equation yields
\begin{eqnarray}
	|\mathbb{E}[\Clip(\langle \mathbf{x}_i, \mathbf{S}_i\rangle S_i(j), R)] - x_i(j)|\leq 2\text{P}(U_{ij}>R-1).\hspace{-5mm}\nonumber\\
\end{eqnarray}
By Hoeffding's inequality,
\begin{eqnarray}
	\text{P}(U_{ij}>t)&\leq& \exp\left(-\frac{t^2}{2\sum_{l\neq j} x_i^2(l)}\right)\nonumber\\
	&\leq & e^{-\frac{t^2}{2m\beta^2}},
\end{eqnarray}
thus
\begin{eqnarray}
	|\mathbb{E}[\hat{\mu}_c(j)]-\mu(j)|\leq 2e^{-\frac{(R-1)^2}{2m\beta^2}}.
\end{eqnarray}

Note that the derivation of the variance remains the same as \autoref{thm:rpc}. Therefore
\begin{eqnarray}
	\mathbb{E}\left[\norm{\hat{\mu}_c-\mu}_2^2\right]&\leq & \frac{d}{n}(V_\epsilon R^2+m)+d\norm{\mathbb{E}[\hat{\mu}_c]-\mu}_\infty^2]\nonumber\\
	&\leq& \frac{d}{n}(V_\epsilon R^2+m)+ 2de^{-\frac{(R-1)^2}{m\beta^2}}.
\end{eqnarray}
By Cauchy-Schwartz inequality,
\begin{eqnarray}
	\mathbb{E}[\norm{\hat{\mu}_c-\mu}_1] \leq d\sqrt{\frac{V_\epsilon R^2+m}{n}} + \sqrt{2} d e^{-\frac{(R-1)^2}{2m\beta^2}}.
\end{eqnarray}
The proof of \autoref{thm:trusted} is complete.

\section{Proof of \autoref{thm:untrusted}}\label{sec:untrusted}
Recall that $\hat{\mu}=(1/n)\sum_{i=1}^n Z_i\mathbf{S}_i$, $\hat{\mu}_c = (1/n)\sum_{i=1}^n Y_i\mathbf{S}_i$. By triangle inequality,
\begin{eqnarray}
	\mathbb{E}\left[\norm{\hat{\mu}-\mu}_1\right] \leq \mathbb{E}\left[\norm{\hat{\mu}_c-\mu}_1\right] + \mathbb{E}\left[\norm{\hat{\mu}-\hat{\mu}_c}\right].
\end{eqnarray}
The first term $\mathbb{E}[\norm{\hat{\mu}_c-\mu}_1]$ is already bounded in \autoref{thm:trusted}. Therefore, now we only bound the second term. Note that
\begin{eqnarray}
	\norm{\hat{\mu}-\hat{\mu}_c}_1 = \norm{\frac{1}{n} \sum_{i=1}^n (Z_i-Y_i)\mathbf{S}_i}_1.
\end{eqnarray}
Follow the analysis from \eqref{eq:attack1} to \eqref{eq:attack2},
\begin{eqnarray}
	&&\mathbb{E}\left[\max_{Z_{1:n}\in \mathcal{Z}} \norm{\hat{\mu}-\hat{\mu}_c}_1\right]\nonumber\\
	&\lesssim& R\left(\frac{q\sqrt{d\ln n}}{n}+\frac{d\sqrt{q}c_\epsilon}{n}\right)\nonumber\\
	&\sim & \beta \sqrt{m\ln n} \left(\frac{q\sqrt{d\ln n}}{n}+\frac{d\sqrt{q}c_\epsilon}{n}\right).
	\label{eq:attack}
\end{eqnarray}
Compared with the bound of $\mathbb{E}[\norm{\hat{\mu}_c-\mu}_1]$ in \eqref{eq:l1general}, the second term in \eqref{eq:attack} does not dominate. Therefore
\begin{eqnarray}
	&&\mathbb{E}\left[\max_{Z_{1:n}\in \mathcal{Z}} \norm{\hat{\mu}-\hat{\mu}_c}_1\right]\nonumber\\
	&&\lesssim q\beta \frac{\sqrt{md\ln^2n}}{n} + d\left(\beta \sqrt{\frac{m \ln n}{n\epsilon^2}}+\sqrt{\frac{m}{n}}\right).
\end{eqnarray}
The proof is complete.

\end{document}